\documentclass[letterpaper]{article}

\PassOptionsToPackage{numbers}{natbib}
\usepackage[preprint]{style}

\usepackage[utf8]{inputenc} % allow utf-8 input
\usepackage[T1]{fontenc}    % use 8-bit T1 fonts
\usepackage{hyperref}       % hyperlinks
\usepackage{url}            % simple URL typesetting
\usepackage{booktabs}       % professional-quality tables
\usepackage{amsfonts}       % blackboard math symbols
\usepackage{nicefrac}       % compact symbols for 1/2, etc.
\usepackage{microtype}      % microtypography
\usepackage{xcolor}         % colors

\usepackage{wrapfig}
\usepackage{amsmath,amssymb}
\usepackage{algorithm}
\usepackage{graphicx}
\usepackage{algpseudocode}
\usepackage{amsthm}
\usepackage{verbatim}
\usepackage{mathtools}
\usepackage{multirow}
\usepackage{subfigure}
\usepackage{empheq}
\usepackage{enumitem}
\usepackage{cases}
\usepackage{array}
\usepackage{accents}
\usepackage{glossaries}

\graphicspath{{../figures}}

\newtheorem{theorem}{Theorem}

\newtheorem{lemma}{Lemma}
\newtheorem{definition}{Definition}
\newtheorem{remark}{Remark}

\newcommand{\rom}[1]{\uppercase\expandafter{\romannumeral #1\relax}}

\newcommand{\fig}{Fig.}
\newcommand{\tab}{Table}
\newcommand{\secR}{Section}
\newcommand{\prop}{Proposition}
\newcommand{\lem}{Lemma}
\newcommand{\theo}{Theorem}

\newcommand{\nl}{\nonumber \\}
\newcommand{\todo}[1]{\color{red}#1\color{black}}

\newcommand{\MultiInstMetric}{\gls{DAP}}

\newcommand{\maptracker}{MapTracker}
\newcommand{\streamy}{StreamMapNet}
\newcommand{\maptrv}{MapTRv2}
\newcommand{\meanLoc}{mLoc.}
\newcommand{\meanDet}{mDet.}

\newcommand{\Loc}{Loc.}
\newcommand{\Det}{Det.}

\newacronym{AP}{AP}{average precision}

\newacronym{FP}{FP}{false positive}
\newacronym{FN}{FN}{false negative}

\newacronym{FD}{FD}{Fr\'{e}chet distance}

\newacronym{CD}{CD}{Chamfer distance}

\newacronym{DP}{DP}{dynamic programming}

\newacronym{GED}{GED}{generalized edit distance}
\newacronym{GOSPA}{GOSPA}{generalized optimal sub-pattern assignment}

\newacronym{HD}{HD}{high-definition}

\newacronym{LIS}{LIS}{longest increasing subsequence}
\newacronym{mAP}{mAP}{mean average precision}
\newacronym{MB}{MB}{multi-Bernoulli}
\newacronym{MOT}{MOT}{multi-object tracking}
\newacronym{MVN}{MVN}{multivariate normal distribution}

\newacronym{NW}{NW}{Needleman--Wunsch}
\newacronym{WF}{WF}{Wagner--Fischer}

\newacronym{OME}{OME}{online map estimation}

\newacronym{O-GOSPA}{O-GOSPA}{ordered-generalized optimal sub-pattern assignment}
\newacronym{OM}{OM}{online mapping}
\newacronym{OSPA}{OSPA}{optimal sub-pattern assignment}

\newacronym{P-GOSPA}{P-GOSPA}{probablistic-GOSPA}

\newacronym{RFS}{RFS}{random finite set}

\newacronym{SOSPA}{SOSPA}{sequence optimal sub-pattern assignment}
\newacronym{SOTA}{SOTA}{state-of-the-art}

\newacronym{TP}{TP}{true positive}
\newacronym{DAP}{PLD}{polyline localisation and detection}

\newacronym{mDAP}{mPLD}{mean-PLD}

\newcommand{\x}{\mathbf{x}} 
\newcommand{\y}{\mathbf{y}}
\newcommand{\z}{\mathbf{z}}

\newcommand{\X}{\mathbf{X}} 
\newcommand{\Y}{\mathbf{Y}}
\newcommand{\Z}{\mathbf{Z}}

\newcommand{\nx}{|\x|} 
\newcommand{\ny}{|\y|}
\newcommand{\nz}{|\z|}

\newcommand{\nX}{|\X|} 
\newcommand{\nY}{|\Y|}

\newcommand{\EnX}{R_{\X}} 
\newcommand{\EnY}{R_{\Y}}

\newcommand{\dogospa}{d^{(c,p)}_{\text{SOSPA}}}

\newcommand{\dNogospa}{\bar{d}^{(c,p)}_{\text{SOSPA}}}

\newcommand{\dcyclicogospa}{d^{(c,p)}_{\text{SOSPA}}}

\newcommand{\dpgospa}{d^{(c,p)}_{\text{PLD}}}
\newcommand{\Spgospa}{S^{(c,p)}_{\text{PLD}}}
\newcommand{\dNpgospa}{\bar{d}^{(c,p)}_{\text{PLD}}}

\newcommand{\dloc}{e^{(c,p)}_{\text{loc.}}}
\newcommand{\dNloc}{\bar{e}^{(c,p)}_{\text{loc.}}}

\newcommand{\dcard}{e_{\text{det.}}}
\newcommand{\dNcard}{\bar{e}_{\text{det.}}}

\newcommand{\dpgospaOne}{d^{(c,1)}_{\text{PLD}}}
\newcommand{\dNpgospaOne}{\bar{d}^{(c,1)}_{\text{PLD}}}

\newcommand{\dlocOne}{e^{(c,1)}_{\text{loc.}}}
\newcommand{\dNlocOne}{\bar{e}^{(c,1)}_{\text{loc.}}}

\newcommand{\thetaStar}{\theta^\star}

\newcommand{\Rsix}{\textrm{R}_{60}}
\newcommand{\Rhundred}{\textrm{R}_{100}}

\newcommand{\SetComposed}{\theta_{\textrm{c}}}

\newcommand{\SetUnordered}{\theta}

\newcommand{\AllAssignSetOrdered}{\Gamma^\textrm{o}}

\newcommand{\Setxz}{\SetUnordered_{\x,\z}}

\newcommand{\Setzy}{\SetUnordered_{\z,\y}}

\newcommand{\Jxz}{J_{\x,\z}}
\newcommand{\Jzy}{J_{\z, \y}}
\newcommand{\JComposed}{J_{\textrm{c}}}

\newcommand{\OpShiftSet}{\mathcal{S}}

\newcommand{\tauAP}{\tau}

\newcommand{\tauFD}{\tau} 

\glsdisablehyper

\author{%
  Chouaib~Bencheikh~Lehocine\textsuperscript{1} \quad
  Adam~Lilja\textsuperscript{1,2} \quad
  Junsheng~Fu\textsuperscript{1} \quad
  Lars~Hammarstrand\textsuperscript{2} \\
  \textsuperscript{1}Zenseact AB \quad \textsuperscript{2}Chalmers University of Technology\\
  Gothenburg, Sweden \\
    \texttt{\{firstname.lastname\}@\{zenseact.com, chalmers.se\}}
}

\begin{document}

\title{Beyond Chamfer Distance: Granular Order-aware Evaluation Metric For Online Mapping
%{%\footnotesize \textsuperscript{*}Note: Sub-titles are not captured in Xplore and
%\thanks{The authors are with $^{1}$Zenseact AB and $^{2}$Chalmers University of Technology, Gothenburg, Sweden (e-mails: \{firstname.lastname\}@\{zenseact.com,chalmers.se\}).}
%\thanks{}
}

%\author{
%Chouaib~Bencheikh~Lehocine$^{1}$ \quad
%Adam Lilja$^{1,2}$ \quad
%Junsheng Fu$^{1}$ \quad
%Lars Hammarstrand$^{2}$ %\\
%$^{1}$Zenseact AB \\
%$^{2}$Chalmers University of Technology\\
%\{firstname.lastname\}@\{zenseact.com,chalmers.se\}
%}

%\author{Chouaib~Bencheikh~Lehocine, AUthor Two, Author Three, Author Four~\IEEEmembership{X,~IEEE}}
%\author{XXX
	%\IEEEauthorblockN{1\textsuperscript{st} Given Name Surname}
%\and
%\and

\maketitle

\begin{abstract}
%Online mapping methods have been advancing rapidly. Nevertheless, their evaluation remains largely dominated by \gls{mAP} and \gls{CD}. While the \gls{mAP}-\gls{CD} framework provides a representative approach for detection evaluation, it suffers from the limited granularity of \gls{CD} and its inability to capture misordered points within geometries. Moreover, the reliance on \gls{mAP} introduces hard decisions on \glspl{TP} and \glspl{FP}, which, although informative for detection performance, fail to reflect positional accuracy.
%In this work, we address these limitations by introducing \gls{SOSPA}. The proposed metric is an order-aware similarity metric for polylines that offers finer granularity than \gls{CD} while also satisfying all metric axioms. We further extend \gls{SOSPA} to closed polylines and establish a result that enables its efficient computation. Building on \gls{SOSPA}, we propose \gls{DAP}, a soft multi-instance \emph{metric} derived from \gls{P-GOSPA} that jointly quantifies detection and positional accuracy. %, with an inherent decomposition into localization and detection error components.
%Through evaluations on online mapping benchmarks, we show that \gls{DAP} preserves the performance trends indicated by \gls{AP} while revealing accuracy differences that \gls{AP} fails to capture, particularly across varying range settings.

Online map estimation is a crucial component of autonomous driving systems that reduces the reliance on costly high-definition maps. \Gls{SOTA} methods commonly predict map elements %in a lightweight, vectorized format 
as ordered sequences of points that form polylines and polygons.
%Despite rapid progress, 
The evaluation of these methods relies predominantly on \gls{mAP} based on thresholded \gls{CD}. This framework lacks sensitivity to point ordering and provides limited granularity in assessing geometric quality, making it difficult to distinguish which methods truly excel over others.
In this work, we address these limitations on two fronts. For the \emph{single-instance similarity measure}, we introduce \gls{SOSPA}, an order-aware \emph{metric} that enables fine-grained evaluation of individual geometries while satisfying all metric axioms. %To accommodate all types of road element geometries, we further extend \gls{SOSPA} to closed polylines. 
For the \emph{multi-instance evaluation framework}, we propose \gls{DAP}, a soft \emph{metric} that jointly captures detection quality and geometric accuracy, replacing the hard thresholding of \gls{mAP} with a principled soft assignment. Through evaluations on nuScenes, we demonstrate that \gls{DAP} effectively ranks \gls{SOTA} online mapping methods (\maptrv, \streamy, \maptracker) while providing a decomposed error analysis. 
This analysis identifies detection capability as the dominant bottleneck in current methods, revealing a performance trend that \gls{mAP} fails to capture. Code for evaluation using our metrics will be released.% upon acceptance.

\end{abstract}

%\begin{IEEEkeywords}
%Online Mapping, Polylines, Evaluation Metrics. %\gls{CD}, \gls{mAP}, GOSPA, GED
%\end{IEEEkeywords}

\glsresetall
\section{Introduction}

Autonomous vehicles require accurate and up-to-date information about the surrounding road geometry to plan safe trajectories. 
\Gls{OME} methods aim to eliminate the reliance on costly pre-built \gls{HD} maps by predicting map elements using on-board sensors in real time. 
The performance of downstream planning and control modules is, in turn, directly dependent on these maps.
Therefore, the evaluation metrics of \gls{OME} methods must reflect detection capability as well as geometric accuracy.
\Gls{SOTA} \gls{OME} methods~\cite{hdmapnet,VectorMapNet2023, maptr,maptrv2, streammapnet, maptracker} represent map elements such as lane dividers, road boundaries, and pedestrian crossings, as ordered sequences of points forming polylines and polygons, each associated with a confidence score and a semantic class label. 
Their evaluation  rests on two pillars: (i)~a \emph{similarity measure} between individual predicted and ground truth geometries, and (ii)~a \emph{multi-instance framework} that aggregates element-wise similarities into an overall score. 
The prevailing evaluation approach employs \gls{CD} for~(i) and \gls{AP} for~(ii), with \gls{mAP} representing the mean score across predicted classes~\cite{hdmapnet,VectorMapNet2023, maptr,maptrv2, streammapnet, maptracker}. 
This \gls{CD}-\gls{AP} combination has two fundamental shortcomings.

\color{black}

\begin{figure}
    \centering
     \includegraphics[width=\textwidth]{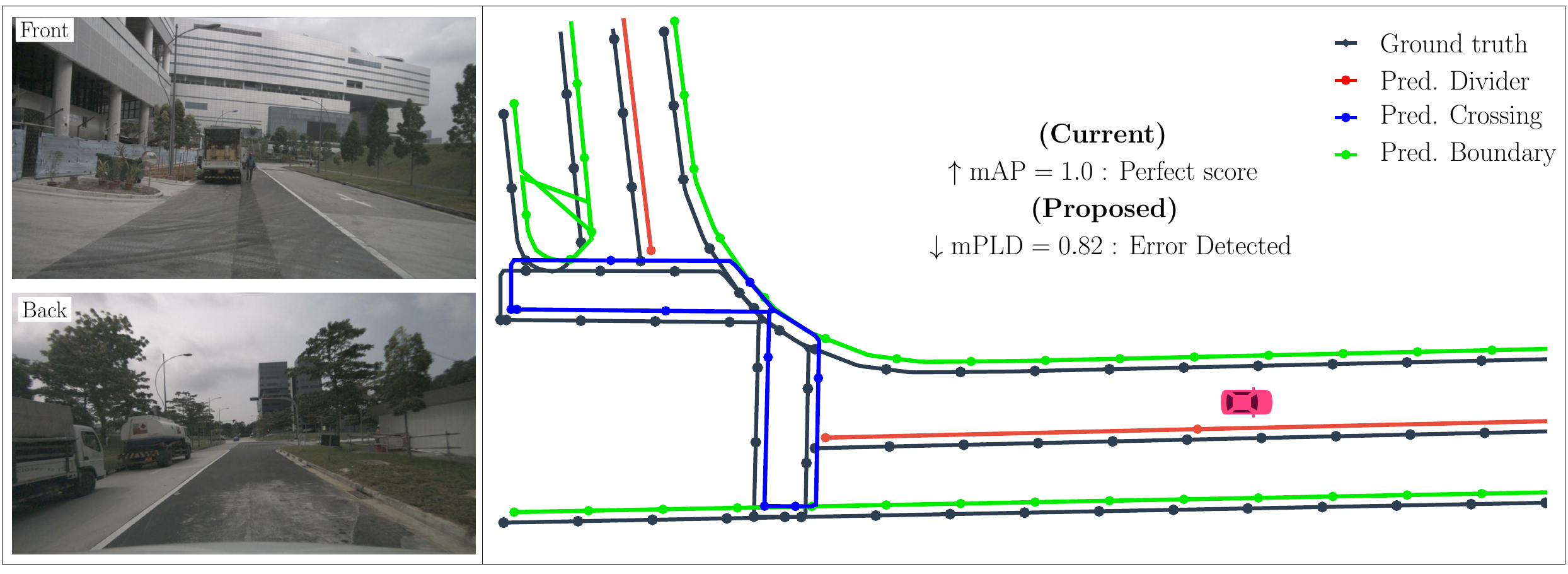}
     
    %\caption{Example showing the limitations of \gls{CD}-\gls{AP}, which reports a maximum score of $1.0$ for predictions that are a laterally shifted version of the ground truth instances, with an offset of $1.5$m. Random order is introduced in one predicted instance. Our proposed metric \MultiInstMetric~ correctly scores the predictions to be far from perfect. (The threshold is $c=2.0$ for both, and $p=1$ is set for \MultiInstMetric. All instances have a confidence of one.)}

    \caption{\gls{CD}-based \gls{mAP} reports a perfect score of 
$1.0$ for predictions diagonally shifted by $1.0$m from the ground truth, with misordering in one instance. Our proposed \gls{mDAP} correctly penalizes such predictions, yielding a score that reflects the geometric degradation. ($\tauAP \in \{1.0, 1.5,2.0\}$ is used for \gls{AP}, and $c=1.5$, $p=1$ for \MultiInstMetric.)}% All confidence scores are one.)}
    
    \label{fig:ap}
 \end{figure}
 
%First, \gls{CD} operates on unordered point sets, making it ill-suited as a similarity measure for polylines where point ordering encodes geometric structure. Beyond this, \gls{CD} has coarse matching granularity as it measures the joint proximity of full instances, and it violates the triangle inequality---hence it is not a metric. 
First, \gls{CD} operates on unordered point sets, making it ill-suited as a similarity measure for polylines where point ordering encodes geometric structure. Beyond this, \gls{CD} measures the joint proximity of full instances, resulting in coarse matching granularity and increased susceptibility to outliers. Additionally, it violates the triangle inequality, rendering it non‑metric. The metric axioms (non-negativity, identity, symmetry, and triangle inequality) ensure that an evaluation measure behaves consistently and intuitively, preventing unexpected outcomes such as two identical geometries being rated as dissimilar, or an indirect comparison path yielding a lower cost than a direct one. %These properties make it ill-suited for polylines, where point ordering encodes geometric structure. 
Some of the limitations of \gls{CD} can be mitigated by adopting \gls{FD}, as was done in~\cite{VectorMapNet2023,OpenLaneV22023}, which is order-aware and constitutes a proper metric.
However, as a supremum-based measure, \gls{FD} reflects only the worst-case pointwise error and is therefore sensitive to outliers and uninformative for the remaining geometry (as shown in \secR~\ref{sec:exp:ome:methods}).

%Second, \gls{AP} thresholds the similarity measure to produce hard \gls{TP}/\gls{FP} decisions. 
Second, \gls{AP} converts the similarity measure into hard \gls{TP}/\gls{FP} decisions by thresholding at a set of distance values (e.g., $\{0.5, 1.0, 1.5\}$\,m), which
%The resulting binary classification 
discards positional accuracy. Thus, two predictions classified as \glspl{TP} at the same threshold contribute equally regardless of their geometric accuracy. Furthermore, the thresholds commonly used are permissive relative to the requirements of downstream autonomous driving applications. %For instance, a threshold of $0.5$\,m can yield near-perfect \gls{AP} scores even when the geometric inaccuracy would critically impact planning and control. 
For instance, even the tightest threshold of $0.5$\,m implies that a prediction that is $0.5$\,m offset can still receive a perfect \gls{AP} score, even though such geometric inaccuracy would critically impact planning and control.
An adequate \gls{OME} evaluation framework should therefore jointly measure detection and geometric accuracy. Furthermore, \gls{AP} itself does not satisfy the metric axioms, making it a non-principled comparison framework.

	To address these shortcomings of existing \gls{OME} \emph{similarity measures} and \emph{multi-instance framework}, we draw inspiration from the \gls{GOSPA} family of metrics from the \gls{MOT} literature~\cite{ospa,gospa, p_gospa}, together with the string matching metric \gls{GED}~\cite{wagner-fischer1974,maes1990cyclic}. For the \emph{similarity measure}, we introduce \gls{SOSPA}, an order-aware adaptation of \gls{GOSPA}~\cite{gospa} from unordered point sets to sequences of points. \gls{GOSPA} is an assignment-based metric that optimally matches individual points between two sets and penalizes unmatched points with a tunable cutoff cost~$c$. This cutoff prevents the metric from being dominated by outliers and admits a direct physical interpretation (e.g., maximum tolerable lateral offset). \gls{SOSPA} inherits these properties while additionally enforcing that matched pairs respect the sequential ordering of points, a key structural property of polylines that \gls{GOSPA} and \gls{CD} both ignore. 
     Additionally, we show that the assignment in \Gls{SOSPA} can be cast as a \gls{GED} minimum-cost assignment problem, enabling its efficient computation via the \gls{WF} algorithm~\cite{wagner-fischer1974} and, %to facilitate the use of \gls{SOSPA} as a matching metric, we 
     provide a bounded normalization in $[0,1]$ that preserves the metric property. Furthermore, to accommodate comparison between different geometries, we extend \gls{SOSPA} to polygons (closed polylines) and establish a mathematical result enabling its efficient computation via the cyclic \gls{GED} algorithm of Maes~\cite{maes1990cyclic}.
     
    %We further cast its computation as a \gls{GED} minimum-cost trace problem solvable via the \gls{WF} algorithm~\cite{wagner-fischer1974}, and provide a bounded normalization in $[0,1]$ that preserves the metric property.  To accommodate comparison between different geometries, we extend \gls{SOSPA} to polygons (closed polylines) and establish a mathematical result enabling its efficient computation via the cyclic \gls{GED} algorithm of Maes~\cite{maes1990cyclic}.

    For the \emph{multi-instance framework}, we propose \gls{DAP}, derived from \gls{P-GOSPA}~\cite{p_gospa}. \gls{P-GOSPA} extends \gls{GOSPA} to sets of objects with associated existence probabilities, performing soft matching that retains geometric distances in the final score. Since online mapping methods output predicted polylines with confidence scores, this formulation naturally applies. \gls{DAP} instantiates \gls{P-GOSPA} with \gls{SOSPA} as the base metric, replacing the hard \gls{TP}/\gls{FP} classification of \gls{AP} with a soft assignment that jointly accounts for the \gls{SOSPA} distance between matched elements and their confidence scores. \Gls{DAP} decomposes into interpretable localization (accuracy) and detection error components, and constitutes a metric. %\fig~\ref{fig:ap} illustrates the contrast where \gls{mDAP}, the mean of \gls{DAP} across map element classes, correctly reflects the degraded prediction quality, whereas \gls{mAP} reports a perfect score.
    An illustrative comparison between our proposed metric and \gls{CD}-\gls{AP} is shown in \fig~\ref{fig:ap}, where the mean of \gls{DAP} across map element classes, \gls{mDAP}, correctly reflects degraded prediction quality, when \gls{mAP} yields a perfect score.

    %\fig~\ref{fig:ap} illustrates the contrast where \gls{mDAP}, the mean of \gls{DAP} across map element classes, correctly reflects the degraded prediction quality, whereas \gls{mAP} reports a perfect score.

    %where \gls{mDAP} denote  mean of \gls{DAP} across map element classes
    
   % An illustrative comparison between our proposed metric and \gls{CD}-\gls{AP} is shown in \fig~\ref{fig:ap}, where \gls{mAP} yields a perfect score whereas \gls{mDAP} correctly reflects the degraded prediction quality.
    
	%\reviewerA{We denote the mean of \gls{DAP} across map element classes as \gls{mDAP}, analogous to \gls{mAP}.} 

Our contributions can be summarized as follows. We
\begin{itemize}
	\item introduce \gls{SOSPA}, an order-aware polyline metric that satisfies all metric axioms, providing finer granularity in detecting matches and mismatches, while being robust to outliers, %We show that \gls{SOSPA} is equivalent to a \gls{GED} minimum-cost trace problem, allowing its efficient computation. We, further, derive  a normalized form bounded in $[0,1]$ that remains a metric for $p=1$,
    \item show that \gls{SOSPA} is equivalent to a \gls{GED} minimum-cost trace problem allowing its efficient computation, and derive a normalized form bounded in $[0,1]$ that remains a metric for $p=1$,
	\item define cyclic \gls{SOSPA} for polygons and establish a lemma enabling its efficient computation, % in $\mathcal{O}(\nx \ny \log \ny)$ time.
	\item propose \gls{DAP}, a soft multi-instance metric jointly quantifying detection and positional accuracy with an inherent decomposition into localization and detection errors, and
	%\item benchmark \gls{SOTA} online mapping methods on nuScenes under the proposed framework, demonstrating that \gls{SOSPA}-\gls{DAP} preserves \gls{CD}-\gls{AP} performance trends while revealing accuracy differences that \gls{AP} cannot capture. Notably, \gls{DAP} error decomposition shows that detection error dominates across all methods, indicating that detection capability remains the primary bottleneck for \gls{OME}.
    \item benchmark \gls{SOTA} online mapping methods on nuScenes~\cite{nuscenes}, showing that \gls{DAP} preserves relative performance rankings indicated by \gls{AP} while its error decomposition reveals that detection error dominates across all methods, a finding that \gls{AP} can not capture.
    
    %a finding that \gls{AP} obscures by discarding positional accuracy.}
\end{itemize}

\color{black}

\section{Related work}\label{sec:related:work}
	\textbf{Online vectorized mapping.}
	\gls{OME} alleviates reliance on \gls{HD} maps in autonomous driving. % by providing a cost-effective way to build maps around vehicles.
    Following the shift from rasterized to vectorized output initiated by HDMapNet~\cite{hdmapnet}, a range of DETR-inspired architectures~\cite{DETR2020} have been proposed for vectorized end-to-end map construction, including VectorMapNet~\cite{VectorMapNet2023}, MapTR~\cite{maptr}, and MapTRv2~\cite{maptrv2}. Temporal fusion has been recently incorporated in methods such as StreamMapNet~\cite{streammapnet} and MapTracker~\cite{maptracker} to further improve consistency and extend detection range. Despite these architectural advancements, evaluation practices, predominantly \gls{AP} with \gls{CD}, have not kept pace, limiting our ability to fully assess progress in this area.

	\textbf{\gls{GOSPA}-family \gls{MOT} metrics.} \Gls{OSPA}~\cite{ospa} introduced a proper metric for \gls{MOT} evaluation over finite sets. \gls{GOSPA}~\cite{gospa} generalized it by removing normalization and allowing flexible penalization of missed and false detections, with subsequent extensions to trajectories~\cite{t_gospa}, graphs~\cite{graphgospa}, and \glspl{RFS} in \gls{P-GOSPA}~\cite{p_gospa}. In particular, \gls{P-GOSPA} incorporates existence probabilities and covariances from \gls{MB} \gls{RFS}, enabling probabilistic soft detection decisions for matches and mismatches. All these \gls{GOSPA}-family metrics share a key merit, which is their decomposition into interpretable error terms associated with matched and unmatched instances. We leverage this property in \gls{DAP}, 
    which is built on \gls{P-GOSPA} with \gls{SOSPA} as its underlying single-instance matching metric. \gls{SOSPA} extends \gls{GOSPA} from unordered sets to ordered sequences and reintroduces normalization to support \gls{OME} instance matching.

	\textbf{Generalized edit distance.}
Comparing ordered sequences has a long history in string matching applications, such as text retrieval and computational biology~\cite{Li2007NormalizedLevenshtein}.
\Gls{GED} (also known as Levenshtein distance)~\cite{levenshtein1966,wagner-fischer1974,maes1990cyclic,Li2007NormalizedLevenshtein} is a metric that quantifies the cost of transforming one string (i.e., sequence) into another through substitution, insertion, and deletion operations.
Normalization into bounded metrics has been studied in~\cite{marzal1993normalizededitdistance,Li2007NormalizedLevenshtein}, while cyclic variants for closed structures were introduced in~\cite{maes1990cyclic}. \Gls{GED} is computed in $\mathcal{O}(\nx \cdot \ny)$ time via dynamic programming using the \gls{WF} algorithm~\cite{wagner-fischer1974}, which solves a minimum-cost trace problem.
Our \gls{SOSPA} formulation can be cast as a minimum-cost trace problem, enabling efficient computation using existing \gls{GED} algorithms.

\section{Order-aware polyline similarity metric: \gls{SOSPA}}

To improve the evaluation of \gls{OME} methods, we address the limitations of existing similarity measures by introducing \gls{SOSPA}, an order-aware similarity metric for finite point sequences, i.e., polylines. %, built upon the \gls{GOSPA} framework~\cite{gospa}.
% In this section, we introduce \gls{SOSPA}, an order-aware similarity metric for finite sequences of points, i.e., polylines, built upon the \gls{GOSPA} framework~\cite{gospa}. %The metric is designed to compare ordered sequences of points---i.e., polylines. 
%We treat two geometrically distinct cases: open polylines, addressed in \secR~\ref{subsec:ogospa}, and polygons (i.e., closed polylines), addressed in \secR~\ref{subsec:ogospa:polygon}.
We treat open polylines (\secR~\ref{subsec:ogospa}) and polygons (\secR~\ref{subsec:ogospa:polygon}) separately and address normalization and computation.

%In this section, we adapt the \gls{GOSPA} framework~\cite{gospa} and introduce \gls{SOSPA} as an assignment-based similarity metric for comparing pairs of polylines.

\subsection{\Glsentrylong{SOSPA}}\label{subsec:ogospa}

%Before defining the \gls{SOSPA} metric, let us define  
We begin by defining $\x = (x_1,\ldots,x_{\nx})$ and 
$\y = (y_1,\ldots,y_{\ny})$ to be finite sequences of points in
$\mathbb{R}^N$. 
An assignment set between $\mathbf{x}$ and $\mathbf{y}$ is any subset %defined as any subset
%\[
$\theta \subseteq \{1,\ldots,\nx\} \times \{1,\ldots,\ny\}$
%\]
satisfying the uniqueness constraint that each index appears at most once. %Formally, for any $(i,j),(i',j') \in \theta$, if $i=i'$ then $j=j'$ and similarly, if $j=j'$ then $i=i'$.
%\[
% $i=i' \implies j=j' \quad \text{and} \quad j=j' \implies i=i'.$
%\]
The collection of all valid assignment sets between $\mathbf{x}$ and $\mathbf{y}$ is denoted by $\Gamma_{\nx,\ny}$.
An assignment set $\theta = \{(i_k, j_k)\}_{k=1}^{|\theta|} \in \Gamma_{\nx,\ny}$ is said to be an \emph{ordered} assignment set if its index pairs respect a strictly increasing order. That is, $i_1 < i_2 < \cdots < i_{|\theta|}$, and $j_1 < j_2 < \cdots < j_{|\theta|}$.
%\begin{align}\label{eq:theta:orderedset}
%i_1 < i_2 < \cdots < i_{|\theta|}, \quad
%j_1 < j_2 < \cdots < j_{|\theta|}. 
%\end{align}
The set of all ordered assignment sets between $\mathbf{x}$ and $\mathbf{y}$ is denoted by $\AllAssignSetOrdered_{\nx,\ny}$.

\begin{definition}\label{definitation:ogospa}
Given a metric $d(.,.)$ in $\mathbb{R}^{N}$, a scalar $c> 0$, and a scalar $p$ where $1\leq p<\infty$, %we define \gls{SOSPA}, between finite sequences $\x$ and $\y$ 
	\begin{align} \label{eq:definition:OGOSPA}
		&\dogospa(\x, \y) =%\nl &
         \Bigg( \min_{\theta \in \AllAssignSetOrdered_{\nx, \ny}} 
		\sum_{(i,j) \in \theta} d(x_i,y_j)^{p}  
		+ 
	 \frac{c^{p}}{2} (\nx + \ny - 2
		|\theta|)
		\Bigg)^{1/p}.
	\end{align}
\end{definition}
From the definition, it follows immediately that \gls{SOSPA} satisfies identity, non-negativity, and symmetry. The proof for the triangle inequality can be found in Appendix~\ref{appendix:sec:ogospa:traingle:proof}.

The main difference between \gls{SOSPA} and \gls{GOSPA}~\cite{gospa} with ($\alpha = 2$) is that \gls{SOSPA} is defined over sequences instead of sets, thus constraining the matches to be ordered
(i.e., $\theta \in \AllAssignSetOrdered$).
%This ordering constraint makes the underlying optimization problem for \gls{SOSPA} fundamentally different from the bipartite matching problem in \gls{GOSPA}, which can be efficiently solved using the Hungarian algorithm~\cite{gospa}. The efficient computation of \gls{SOSPA} is discussed in the next subsection.
The cutoff distance $c$ in~\eqref{eq:definition:OGOSPA} defines the maximum allowable error for pairs to be regarded as matched points, as well as the penalty for missed and false detections given by $c^p/2$. 
The exponent $p \geq 1$ determines the sensitivity of the metric to localization errors, where a higher $p$ results in higher penalization of large deviations among matched pairs. The use of a constant penalty $c^p/2$ per unmatched point ensures the robustness of the metric to outliers, since they cannot excessively influence the final score. Furthermore, while the overall matching must respect the sequential order of points in both sequences, the formulation allows individual points in each sequence to be skipped (i.e., left unmatched) for $c^p/2$ per skipped point, providing flexibility in handling sequences of unequal length, or locally poor correspondence.

%The cutoff distance $c$ in~\eqref{eq:definition:OGOSPA} defines the maximum allowable error for matched points, as well as the penalty for missed and false detections given by $c^p/2$. 
 %The exponent $p \geq 1$ determines the sensitivity of the metric to outliers, where a higher $p$ results in higher penalization of outliers.
%The use of a constant penalty $c^p/2$ per unmatched point ensures the robustness of the metric to outliers, since they cannot excessively influence the final score. Furthermore, while the overall matching must respect the sequential order of points in both sequences, the formulation allows individual points in each sequence to be skipped (i.e., left unmatched) for $c^p/2$ per skipped point, providing flexibility in handling sequences of unequal length, or locally poor correspondence.

%\subsection{Computing \gls{SOSPA}}
\textbf{Computing \gls{SOSPA}.} The ordering constraint in~\eqref{eq:definition:OGOSPA} makes the underlying optimization problem for \gls{SOSPA} fundamentally different from the bipartite matching problem in \gls{GOSPA}, which can be efficiently solved using the Hungarian algorithm.
To find an efficient solution to~\eqref{eq:definition:OGOSPA}, we propose to cast it as a minimum-cost trace problem from the \gls{GED} literature~\cite{wagner-fischer1974, marzal1993normalizededitdistance}. \gls{GED} measures the minimum cost of transforming a sequence $\x$ into $\y$ through (i) substitutions $x_i\!\rightarrow\! y_j$ at cost $\gamma(x_i\!\rightarrow\! y_j)\geq 0$, (ii) deletions $x_i\!\rightarrow\!\Lambda$ at cost $\gamma(x_i\!\rightarrow\!\Lambda)\geq 0$, and (iii) insertions $\Lambda\!\rightarrow\! y_j$ at cost $\gamma(\Lambda\!\rightarrow\! y_j)\geq 0$, where $\Lambda$ is the null element. A \emph{trace} is a set $T$ of index pairs that respects sequence ordering, i.e., an ordered assignment set. Its cost is
%Given two sequences $\x$ and $\y$, the edit distance measures the cost of transforming $\x$ into $\y$ through elementary operations with associated non-negative costs. These are : (i) \emph{substitution} which is replacing $x_i$ with $y_j$ at cost $\gamma(x_i \rightarrow y_j) \geq 0$; (ii) \emph{deletion} which is removing $x_i$ at cost $\gamma(x_i \rightarrow \Lambda) \geq 0$, and (iii) \emph{insertion} which corresponds to inserting $y_j$ at cost $\gamma(\Lambda \rightarrow y_j) \geq 0$, 
%\begin{enumerate}[label=(\roman*)]
%	\item \emph{Substitution}: replacing $x_i$ with $y_j$ at cost $\gamma(x_i \rightarrow y_j) \geq 0$,
%	\item \emph{Deletion}: removing $x_i$ at cost $\gamma(x_i \rightarrow \Lambda) \geq 0$,
%	\item \emph{Insertion}: inserting $y_j$ at cost $\gamma(\Lambda \rightarrow y_j) \geq 0$,
%\end{enumerate}
%where $\Lambda$ denotes the null element~\cite{wagner-fischer1974, marzal1993normalizededitdistance}.
%A \emph{trace} is a set $T$ of index pairs that respects sequence ordering, i.e., an ordered assignment set satisfying~\eqref{eq:theta:orderedset}. The cost of a trace $T$ is
\begin{align}\label{eq:trace:cost:alt}
    C(T) &= \sum_{(i,j) \in T} \gamma(x_i \rightarrow y_j) + \sum_{i \notin \pi_1(T)} \gamma(x_i \rightarrow \Lambda) +%\nl &\qquad 
   \sum_{j \notin \pi_2(T)} \gamma(\Lambda \rightarrow y_j),
\end{align}
where $\pi_1(T) = \{i : (i,j) \in T\}$ and $\pi_2(T) = \{j : (i,j) \in T\}$ are index projections of~$T$.
We now show that \gls{SOSPA} is equivalent to a minimum cost trace of \gls{GED} with specific cost functions.

\begin{lemma}\label{lemma:ogospa:edit:equivalence:alt}
%The \gls{SOSPA} distance raised to the power $p$ is equivalent to the minimum cost trace of \gls{GED} between $\x$ and $\y$ with cost functions
$\dogospa(\x,\y)^p$ equals the minimum trace cost of \gls{GED} between $\x$ and $\y$ with cost functions
\begin{align}\label{eq:ogospa:operation:costs:alt}
    \gamma(x_i \rightarrow y_j) &= d(x_i, y_j)^p, \quad %\nl
    \gamma(x_i \rightarrow \Lambda) = \gamma(\Lambda \rightarrow y_j) = \frac{c^p}{2}.
\end{align}
\end{lemma}

\begin{proof}
Substituting~\eqref{eq:ogospa:operation:costs:alt} into~\eqref{eq:trace:cost:alt} and noting that traces are ordered assignment sets, the minimum trace cost, $\min_T C(T)$, coincides with~\eqref{eq:definition:OGOSPA} raised to the power $p$.
\end{proof}

By \lem~\ref{lemma:ogospa:edit:equivalence:alt}, \gls{SOSPA} can be computed via \gls{DP} using the \gls{WF} algorithm~\cite{wagner-fischer1974} (equivalently, Needleman-Wunsch algorithm~\cite{needleman-wunsch1970}) in $\mathcal{O}(\nx \cdot \ny)$ time.

\begin{remark}\label{remark:ogospa:metric:via:ged}
When $\gamma(\cdot \rightarrow \cdot)$ is a metric, \gls{GED} coincides with the minimum trace cost $\min_T C(T)$~\cite{wagner-fischer1974} and is a metric. The costs in~\eqref{eq:ogospa:operation:costs:alt} satisfy these conditions for $p = 1$ (since $d(\cdot,\cdot)^p$ is a metric when $p = 1$), providing an alternative proof of the triangle inequality for \gls{SOSPA} (see Appendix for the direct proof covering all $p \geq 1$). Given $\gamma(\cdot \rightarrow \cdot)$ is not a metric, the minimum-cost trace problem remains well-defined for any non-negative cost function, and the \gls{WF} algorithm solves it regardless. Hence, the \gls{WF} algorithm correctly computes \gls{SOSPA} for all $p \geq 1$.
\end{remark}

%The \gls{GED} framework admits normalization procedures~\cite{marzal1993normalizededitdistance, Li2007NormalizedLevenshtein}, which we exploit in the next subsection to obtain a normalized \gls{SOSPA} metric bounded in $[0,1]$.

%\subsection{Normalization of \gls{SOSPA}}
\textbf{Normalization of \gls{SOSPA}.}
To facilitate the use of \gls{SOSPA} as a polyline matching metric, we normalize $\dogospa$ while preserving its metric property. A natural normalization is $\dogospa/\big(\tfrac{c^p}{2}(\nx+\ny)\big)^{1/p}$. %$\dogospa(\x, \y)/ \big(\frac{c^p}{2}(\nx + \ny)\big)^{1/p}$ %One approach to normalize the metric is
%\begin{align}
	%\dNogospaSum(\x, \y) = 
 % 0\leq   \frac{ \dogospa(\x, \y)}{\big(\frac{c^p}{2}(\nx + \ny)\big)^{1/p}} \leq 1.
%\end{align}
However, this does not satisfy the triangle inequality in general. Instead, following the approach used to normalize \gls{GED} in~\cite{Li2007NormalizedLevenshtein}, we define
$\dNogospa$ as follows.
\begin{definition}\label{definitation:normalized:ogospa}
        Let $\dogospa$ be as defined in~\eqref{eq:definition:OGOSPA}. Then %normalized \gls{SOSPA} is
		\begin{align} \label{eq:normalized:ogospa}
			\dNogospa(\x, \y) = \frac{2 \cdot \dogospa(\x,\y)}{\big(\frac{c^p}{2}(\nx + \ny)\big)^{1/p} + \dogospa(\x,\y)}  \in [0,1].
		\end{align}
		%yielding $0 \leq \dNogospa(\x, \y) \leq 1$.
	\end{definition}
	For $p=1$, $\dNogospa$ satisfies the triangle inequality as shown in Appendix~\ref{appendix:sec:normalization:proof}, and is thus a metric (identity and symmetry follow from $\dogospa$).

\subsection{\gls{SOSPA} for polygons} \label{subsec:ogospa:polygon}
Since the starting point of a polygon (i.e., closed polyline) can be changed without altering its geometry, computing \gls{SOSPA} for polygons (i.e., closed polylines) requires inspecting all cyclic shift combinations of the polygons. The same difficulty arises for cyclic \gls{FD}~\cite{frechetpolygon} and cyclic \gls{GED}~\cite{maes1990cyclic}.
Therefore, in the following, we define \gls{SOSPA} over cyclic sequences and establish a \lem~that enables solving it using an efficient exact algorithm that leverages cyclic shift properties.

%%\textcolor{blue}{or affecting the relative order between points, } %computing \gls{SOSPA} for closed polylines (i.e., polygons) requires careful treatment. This is also the case when computing other order-dependent distances, such as the Fréchet distance~\cite{frechetpolygon} or \gls{GED}~\cite{maes1990cyclic} for cyclic structures.

%We begin by introducing some notation. For convenience, we use 
Let $\langle k \rangle_n$ denote a cyclic shift operation following 
\begin{align}\label{eq:cyclic:bracket:notation}
    \langle k \rangle_n = ((k - 1) \mod n) + 1, \quad k \in \mathbb{Z}, \; n \in \mathbb{N}^+.
\end{align}
Then, define the cyclically shifted version of the sequence $\x$ as
\begin{align}\label{eq:cyclic:shift:set}
    \OpShiftSet_s(\x) = \big( x_{\langle 1+s \rangle_{\nx}}, x_{\langle 2+s \rangle_{\nx}}, \ldots, x_{\langle \nx +s \rangle_{\nx}} \big), \quad  s \in \mathbb{Z}.
\end{align}
We can now readily denote the set of all cyclic shifts of a finite sequence as  $[\x] = \{ \OpShiftSet_{l}(\x): 0 \leq l\leq \nx -1 \}$, and this will be referred to as a cyclic sequence.

\begin{definition}\label{definition:ogospa:closed}
Given that two finite cyclic sequences $[\x]$ and $[\y]$, then \gls{SOSPA} is defined as
     \begin{align}\label{eq:definition:ogospa:cyclic}
    \dcyclicogospa([\x], [\y]) = \min_{\substack{s_x \in \{0,\ldots, \nx -1 \} \\ s_y \in \{0,\ldots, \ny -1 \}}}
    \dogospa\big(\OpShiftSet_{s_x}(\x), \OpShiftSet_{s_y}(\y)\big).
\end{align}
\end{definition}

A straightforward evaluation of~\eqref{eq:definition:ogospa:cyclic} requires $\nx\ny$ computations of \gls{SOSPA}.
Following a key result for cyclic \gls{GED}~\cite[\lem~3.1]{maes1990cyclic}, we show that optimizing over the cyclic shifts of one sequence while fixing the other suffices.
%we establish that to compute cyclic \gls{SOSPA}, it is enough to optimize over all cyclic shifts of a single cyclic sequence while fixing the other.
\color{black}
\begin{lemma}\label{lemma:maes:cyclic}
For any two finite cyclic sequences $[\x]$ and $[\y]$,
we have
	\begin{align}\label{eq:maes:result}
		\dcyclicogospa([\x], [\y]) &= \dcyclicogospa(\x, [\y]) %\nl&
        =\min_{s_y \in \{0,\ldots,\ny-1\}} \dogospa\big(\x, \OpShiftSet_{s_y}(\y)\big).
	\end{align}
	%That is, one can fix any representative of $\x$ and optimize only over cyclic shifts of $\y$ and vice versa.
\begin{proof}
The proof follows a similar argument to one outlined in~\cite[\lem~3.1]{maes1990cyclic}. See Appendix~\ref{appendix:sec:cyclic:proof} for details.
\end{proof}
\end{lemma}

Using the result above, and the established relation between \gls{GED} and \gls{SOSPA} in \lem~\ref{lemma:ogospa:edit:equivalence:alt}, the cyclic \gls{SOSPA} can be computed 
efficiently using Maes algorithm~\cite{maes1990cyclic} in $\mathcal{O}(\nx \ny \log \ny )$ time, and faster cyclic \gls{GED} algorithms such as~\cite{marzal2000EditDistance} directly apply. We note that the original reference on cyclic \gls{GED}~\cite{maes1990cyclic} does not establish whether cyclic \gls{GED} constitutes a metric (the same is true for cyclic \gls{FD}~\cite{frechetpolygon}). As cyclic \gls{SOSPA} shares a similar structural formulation with cyclic \gls{GED}, we likewise do not claim it to be a metric, as we believe it does not satisfy the triangle inequality in general.

\color{black}
\color{black}

\section{Multi-instance evaluation metric: \MultiInstMetric}

To score \gls{OME} methods, we need a multi-instance evaluation measure. While \gls{SOSPA} can be combined with \gls{AP}, this combination still inherits \gls{AP} rigid matching strategy and discards the geometric accuracy. Therefore, in this section, we introduce \gls{DAP}, a soft multi-instance metric that jointly evaluates detection and geometric accuracy.

\subsection{Online mapping outputs as multi-Bernoulli RFS}
For each road element class, an \gls{OME} method outputs %a collection of predicted instances modeled by 
$\X = \big\{ (r_i, \x_i) \big\}_{i=1}^{\nX}, $%. Each instance is described by a confidence score and a polyline geometry following
%\begin{align}\label{eq:polylineWithExistnecProb}
 %   \X = \big\{ (r_i, \x_i) \big\}_{i=1}^{\nX}
%\end{align}
where $r_i \in [0,1]$ is the confidence score of the $i$-th instance and $\x_i$ is a polyline (i.e., a finite sequence of points in $\mathbb{R}^N$) describing its estimated geometry. The output $\X$ parametrizes a \gls{MB} \gls{RFS} distribution over sequences, where each component $(r_i, \x_i)$ constitutes an independent Bernoulli process with existence probability $r_i$ and a degenerate single-object density concentrated at the polyline $\x_i$. Ground truth instances take the same form with $r_i = 1$ for all $i$. This probabilistic interpretation of \gls{OME} outputs promotes the use of \gls{P-GOSPA}~\cite{p_gospa}, a metric designed for comparing \gls{MB} distributions, as a multi-instance evaluation framework. Accordingly, we define \MultiInstMetric\ as the instantiation of \gls{P-GOSPA} that uses $\dNogospa$ as the base single-instance distance, and adopt it as our multi-instance evaluation metric.
%{\color{magenta}Since online mapping outputs naturally parameterize \gls{MB} \gls{RFS} distributions, \gls{P-GOSPA}~\cite{p_gospa}---a metric designed for comparing \gls{MB} distributions---provides the natural multi-instance evaluation framework. We define \MultiInstMetric\ as the instantiation of \gls{P-GOSPA} that uses $\dNogospa$ as the base single-instance distance.}
%\gls{P-GOSPA} jointly penalizes localization errors, missed detections, and false alarms, while weighting each contribution by the respective existence probabilities.

\color{black}

\subsection{\Glsentrylong{DAP} error metric}

%{\color{blue}
%Building on the \gls{MB} \gls{RFS} interpretation of online mapping baselines output, 
%In the following, we define \MultiInstMetric\ as an adaptation of \gls{P-GOSPA}~\cite{p_gospa} that uses $\dNogospa$ as the base single-instance metric.
%}

\begin{definition}\label{definition:pgospa}
	For sets of polylines $\X$ and $\Y$, $c>0$ and $1\leq p<\infty$, %and let $p, c$ be scalers, where $0<c$ and $1\leq p<\infty$, then we define

	\begin{align}\label{eq:definition:PGOSPA}
		 \dpgospa(\X, \Y) &= %\nl&
		\Bigg[ \min_{\SetUnordered \in \Gamma} 
		\Bigg(
		\sum_{(i,j) \in \SetUnordered} 
		\left[
		\min(r_{\x_i}, r_{\y_j}) \, \dNogospa(\x_i, \y_j)^p 
+\frac{1}{2} \left|r_{\x_i} - r_{\y_j} \right|
		\right]  \nl
		&  \qquad + \frac{1}{2} 
		\bigg(
		\sum_{\substack{i : \forall j, (i,j) \notin \SetUnordered}} r_{\x_i}
		+ \sum_{\substack{j : \forall i, (i,j) \notin \SetUnordered}} r_{\y_j}
		\bigg)
		\Bigg)
		\Bigg]^{\frac{1}{p}}.
	\end{align}
	
		where $\dNogospa$ is defined in~\eqref{eq:normalized:ogospa}, and where $\SetUnordered \in \Gamma$ is an assignment set between $\X$ and $\Y$.
\end{definition}

%Note that this definition is the same as the \gls{P-GOSPA} metric~\cite[\prop~2]{p_gospa} where $\alpha = 2$, $c = 1$, {\color{blue}and  $\dNogospa$ as the base distance metric. }% We use polylines with existence probability rather than \gls{MB} densities, and $\dNogospa$ as the distance measure between polylines.By~\cite[\prop~2]{p_gospa}  and the fact that $\dNogospa$ is metric for $p=1$, it follows that \MultiInstMetric\ is a metric for $p=1$. \MultiInstMetric\ can be easily solved via a 2D assignment problem~\cite{p_gospa}, using the Hungarian algorithm or other linear sum assignment solvers.

This is \gls{P-GOSPA}~\cite[\prop.~2]{p_gospa} with $\alpha=2$, $c=1$, and $\dNogospa$ as the base distance. Since $\dNogospa$ is a metric for $p=1$, so is \MultiInstMetric~by~\cite[\prop.~2]{p_gospa}. The minimization is a 2D assignment problem, solvable in polynomial time by the Hungarian algorithm.

Looking at~\eqref{eq:definition:PGOSPA}, the role of confidence scores becomes clear. Unmatched instances incur a penalty proportional to their confidence. A high-confidence miss or false alarm is penalized more than an uncertain one, reflecting the intuition that confident errors are more consequential. For matched pairs, the cost combines a confidence-weighted geometric term %scaled by $\min(r_{\x_i}, r_{\y_j})$ so that uncertain matches contribute less to the localization error, 
with a confidence-mismatch term $\tfrac{1}{2}|r_{\x_i} - r_{\y_j}|$, which penalizes disagreement between predicted and ground truth existence probabilities. Consequently, a good match requires not only geometric accuracy but also well-calibrated confidence. This makes \MultiInstMetric\ a comprehensive measure for evaluating probabilistic map instances as produced by online mapping methods. Although ground-truth carries $r_j=1$, \eqref{eq:definition:PGOSPA} is fully general and supports comparing two stochastic outputs (e.g.\ for measuring temporal consistency between consecutive map predictions).
%\footnote{Although ground truth instances carry $r_j = 1$, the formulation~\eqref{eq:definition:PGOSPA} is general and permits both $\X$ and $\Y$ to carry non-trivial existence probabilities. This enables applications beyond standard predictions Vs. ground truth benchmarking, such as measuring the temporal consistency of predictions by comparing map outputs at successive time steps.}

%\begin{remark}
%{\color{blue}
%Although ground truth instances carry $r_j = 1$, the formulation~\eqref{eq:polylineWithExistnecProb} is general and permits both $\X$ and $\Y$ to carry non-trivial existence probabilities. This enables applications beyond standard benchmarking, such as measuring the temporal consistency of predictions by comparing map outputs at successive time steps.
%}
%\end{remark}

%Looking at~\eqref{eq:definition:PGOSPA} one can easily see the effect of confidence. For unmatched instances, \MultiInstMetric\ penalizes these proportionally to their confidence score. The more confident the instances are, the higher the error is. As for matched instances, their cost depends on the geometrical error as well as the difference in confidence score. The best matches does not only need to be geometrically accurate but they need to correctly match the expected confidence. Therefore, \MultiInstMetric\ provides a comprehesive measures of probabilistic instances as those output by online mapping frameworks.

%\subsection{\MultiInstMetric\ decomposition into interpretable error terms}

\textbf{\MultiInstMetric\ decomposition into interpretable error terms.} Letting $\thetaStar$ denote the optimal assignment in~\eqref{eq:definition:PGOSPA}, define
\begin{align}
\dloc(\X,\Y) &= \!\!\sum_{(i,j)\in\thetaStar}\!\!\min(r_{\x_i},r_{\y_j})\,\dNogospa(\x_i,\y_j)^p,\label{eq:dap:loc}\\
\dcard(\X,\Y) &= \!\!\sum_{(i,j)\in\thetaStar}\!\!\tfrac{1}{2}|r_{\x_i}-r_{\y_j}| + \tfrac{1}{2}\Big(\!\sum_{i:(i,j)\notin\thetaStar}\!\! r_{\x_i} + \!\sum_{j:(i,j)\notin\thetaStar}\!\! r_{\y_j}\Big),\label{eq:dap:card}
\end{align}
so that $\dpgospa(\X,\Y)=[\dloc(\X,\Y)+\dcard(\X,\Y)]^{1/p}$. The term $\dloc$ is a probability-weighted aggregation of geometric error over matched pairs. The quantity $\dcard$ pools detection-related penalties: confidence mismatch on matched pairs, missed instances, and false alarms. Including confidence mismatch within the detection term is justified by the fact that a matched instance with confidence $r$ behaves as a $(1-r)$-weighted false detection. The constant $\tfrac{1}{2}$ on detection penalties keeps the metric robust to outliers.

\textbf{Normalization of \MultiInstMetric.}
To facilitate comparability across categories and baselines, we apply the same normalization as used for \gls{SOSPA}. That is,
\begin{align}\label{eq:dap:normalized}
\dNpgospa(\X,\Y) = \frac{2\,\dpgospa(\X,\Y)}{\big(\tfrac{1}{2}(\EnX+\EnY)\big)^{1/p} + \dpgospa(\X,\Y)} \in [0,1],
\end{align}
where $\EnX=\sum_i r_{\x_i}$ and $\EnY=\sum_j r_{\y_j}$.  The normalized \MultiInstMetric\ is positive and symmetric. Moreover, for the case $p=1$, it does satisfy the triangle inequality as shown in Appendix~\ref{appendix:sec:normalization:proof} and is therefore a metric. Beyond satisfying the triangle inequality, the case $p=1$ enables a particularly convenient use of the decomposition terms. Specifically,~\eqref{eq:definition:PGOSPA} simplifies to $\dpgospaOne=\dlocOne+\dcard$, and the decomposition extends naturally to the normalized form with
\begin{align}\label{eq:dap:loc:card:normalized}
\dNlocOne(\X,\Y) &= \frac{2\,\dlocOne(\X,\Y)}{\tfrac{1}{2}(\EnX+\EnY)+\dpgospaOne(\X,\Y)},\quad
\dNcard(\X,\Y) = \frac{2\,\dcard(\X,\Y)}{\tfrac{1}{2}(\EnX+\EnY)+\dpgospaOne(\X,\Y)},
\end{align}
giving $\dNpgospaOne=\dNlocOne+\dNcard$. We use $p=1$ throughout the experiments.

%\subsection{Scene-Level Aggregation}
%Let $k$ represent a scene index, and let $\{\X\}_{k=1}^{K}$, $\{\Y\}_{k=1}^{K}$, be sets of polylines over $K$ scenes. We can readily compute $\dpgospa$ between the two sets, following
%\begin{align}
%	\dpgospa(\{\X\}_{k=1}^{K}&, \{\Y\}_{k=1}^{K})
%	=
%	 \sum_{k=1}^{K}\dpgospa(\X_k,\Y_k)
%\end{align}
%where $k$ indexes the scenes. Normalization following~\eqref{eq:normalized:pgospa} can be applied afterwards.

\section{Experiments}
In this section, we highlight the advantages of \MultiInstMetric\ and \gls{SOSPA} using illustrative examples. We then benchmark three \gls{SOTA} \gls{OME} methods on nuScenes~\cite{nuscenes}, the most widely adopted benchmark for \gls{OME}. %The decomposition terms $\dloc$ and $\dcard$ will be referred to as \Loc{} and \Det{} in the tables.

%\begin{wrapfigure}{r}{0.6\textwidth}
\begin{figure}
    \centering
    \subfigure[\Glspl{FP} granularity]{\includegraphics[angle=0,width=0.20\textwidth]{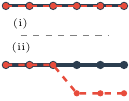}\label{fig:toy12:sub1}}
\hfill
\subfigure[\Glspl{TP} granularity]{\includegraphics[angle=0,width=0.20\textwidth]{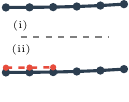} \label{fig:toy12:sub2}}
\hfill
\subfigure[Erroneous point order]{\includegraphics[angle=0,width=0.20\textwidth]{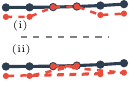}\label{fig:toy5:order}}
    
\caption{Illustrative examples highlighting the advantages of \MultiInstMetric. All confidences set to $1$.}% (see \tab~\ref{tab:toy_results:main}.)}
     \label{fig:toy:all}
     
%\end{wrapfigure}
\end{figure}
\begin{table}[!t]
\centering
\caption{\MultiInstMetric\ ($c = 0.5, p= 1$) versus \gls{CD}/\gls{FD}-\gls{AP} ($\tauAP = 0.5$) on the illustrative examples in \fig~\ref{fig:toy:all}}

\renewcommand{\arraystretch}{0.9}
\setlength{\tabcolsep}{3pt}
\begin{tabular}{ c c c c c c c }
\toprule
&\multirow{2}{*}{Case} 
  & \multicolumn{3}{c}{\MultiInstMetric } 
  & \gls{CD}-\gls{AP} & \gls{FD}-\gls{AP}\\
\cmidrule(lr){3-5}\cmidrule(l){6-6}\cmidrule(l){7-7}
& & val. & \Loc & \Det & val. & val. \\
\midrule
\multirow{2}{*}{\glspl{FP} granularity (\fig~\ref{fig:toy12:sub1})}
&(i)  & $0.0$   & $0.0$   & $0.0$   & $1.0$ & $1.0$\\
&(ii)  & $0.800$ & $0.0$   & $0.800$ & $1.0$  & $0.0$\\
\midrule
\multirow{2}{*}{\glspl{TP} granularity (\fig~\ref{fig:toy12:sub2})}
& (i)   & $1.000$ & $0.0$ & $1.000$ & $0.0$ & $0.0$ \\
&(ii)  & $0.778$ & $0.222$ & $0.556$ & $0.0$  & $0.0$ \\
\midrule
\multirow{2}{*}{Erroneous Order (\fig~\ref{fig:toy5:order})} 
&(i)  & $0.593$ & $0.593$ & $0.0$ & $1.0$ & $1.0$ \\
&(ii)  & $0.778$ & $0.222$ & $0.556$ & $1.0$ &  $0.0$ \\
\bottomrule
\end{tabular}

\label{tab:toy_results:main}
\end{table}

\subsection{Results from illustrative examples}

\fig~\ref{fig:toy:all} shows three scenarios specifically constructed to highlight limitations of \gls{CD}/\gls{FD}-\gls{AP}, with results reported in \tab~\ref{tab:toy_results:main}.
For \fig~\ref{fig:toy12:sub1}, \gls{CD}-\gls{AP} returns a perfect score for \emph{both} predictions, even though prediction~(ii) includes spurious points. \gls{DAP} correctly assigns prediction~(ii) a higher error, attributed entirely to $\dNcard$ (\Det{}). \gls{FD}-\gls{AP} also detects the false positive points, but due to its supremum-based formulation, it excessively penalizes the prediction, resulting in $\text{AP}=0$. For \fig~\ref{fig:toy12:sub2}, \gls{CD}-\gls{AP} fails to distinguish between a partial match and no match at all, both yielding $\text{AP}=0$, as does \gls{FD}-\gls{AP}. In contrast, \gls{DAP} ranks the partial match higher and, through its error decomposition, captures both the partial matching and the associated accuracy error.
For \fig~\ref{fig:toy5:order}, prediction~(ii) has the same point positions as prediction~(i), but with a scrambled order. \gls{CD}-\gls{AP} reports $1.0$ in both cases. \gls{FD}-\gls{AP} correctly detects the misordering, but penalizes it heavily yielding $\text{AP}=0$. \gls{DAP} not only detects the misordering but also penalizes it appropriately, assigning $\text{\gls{DAP}} < 1.0$. The decomposition is particularly insightful in this case, as $\dNcard=0$ for the correct ordering and $\dNcard>0$ for the scrambled one. These observations confirm that, by construction, \gls{DAP} provides finer granularity in identifying \glspl{TP} and \glspl{FP}, and is sensitive to misordering, both merits inherited from its \gls{SOSPA} base.

\color{black}

\begin{table}[!t]
\centering
\caption{\gls{AP} results for different baselines based on both \gls{CD} and \gls{FD}.}
\label{tab:allmethods:ap}
	
	\setlength{\tabcolsep}{2pt}
	\begin{tabular}{l c c c c c c c c c}
		\toprule
		& \multirow{2}{*}{\textbf{Baseline}}
		& \multicolumn{4}{c}{$\uparrow$~\gls{CD}-\gls{AP}}
		& \multicolumn{4}{c}{$\uparrow$~\gls{FD}-\gls{AP}} \\
		\cmidrule(lr){3-6}\cmidrule(l){7-10}
		& & Crossing & Divider & Boundary & \textbf{mAP}\footnotemark 
		& Crossing & Divider & Boundary & \textbf{mAP} \\
		\midrule
		
		\multirow{3}{*}{$\Rsix$}
		& \maptrv{}      & $0.127$ & $0.262$ & $\mathbf{0.445}$ & $0.278$ & $0.052$ & $0.107$ & $0.292$ & $0.150$ \\
		& \streamy{}     & $0.316$ & $0.283$ & $0.407$ & $0.335$ & $0.150$ & $\mathbf{0.132}$ & $0.257$ & $0.180$ \\
		& \maptracker{}  & $\mathbf{0.441}$ & $\mathbf{0.288}$ & $0.440$ & $\mathbf{0.390}$ & $\mathbf{0.242}$ & $0.130$ & $\mathbf{0.295}$ & $\mathbf{0.222}$ \\
		\midrule
		
		\multirow{3}{*}{$\Rhundred$}
		& \maptrv{}      & $0.081$ & $0.174$ & $0.293$ & $0.183$ & $0.017$ & $0.020$ & $0.051$ & $0.029$ \\
		& \streamy{}     & $0.249$ & $0.167$ & $0.241$ & $0.219$ & $0.036$ & $0.020$ & $0.036$ & $0.031$ \\
		& \maptracker{}  & $\mathbf{0.455}$ & $\mathbf{0.235}$ & $\mathbf{0.381}$ & $\mathbf{0.357}$ & $\mathbf{0.147}$ & $\mathbf{0.040}$ & $\mathbf{0.085}$ & $\mathbf{0.091}$ \\
		\bottomrule
	\end{tabular}
	\vspace{-1em}
\end{table}

\begin{comment}
    =================================Legacy=====================
\begin{table}[!t]
\centering
\caption{\gls{mAP} results for different baselines. $\Rsix = 60\times30$ m $\Rhundred = 100\times50$ m.}
\label{tab:allmethods:ap}

\setlength{\tabcolsep}{3pt}
\begin{tabular}{l c c c c c}
\toprule
& \textbf{Baseline} & Crossing & Divider & Boundary & \textbf{mAP}\footnotemark $\uparrow$ \\
\midrule

\multirow{3}{*}{$\Rsix$}
& Maptrv2        & $0.127$ & $0.262$ & $\mathbf{0.445}$ & $0.278$ \\
& StreamMapNet   & $0.309$ & $0.274$ & $0.397$ & $0.327$\\
&MapTracker     & $\mathbf{0.441}$ & $\mathbf{0.288}$ & $0.440$ & $\mathbf{0.390}$ \\
\midrule

\multirow{3}{*}{$\Rhundred$}
& Maptrv2        & $0.081$   &  $0.174$   & $0.293$   & $0.182$   \\
& StreamMapNet   & $0.244$ & $0.163$ & $0.236$ & $0.214$ \\
& MapTracker     & $\mathbf{0.455}$ & $\mathbf{0.235}$ & $\mathbf{0.381}$ & $\mathbf{0.357}$ \\
\bottomrule
\end{tabular}

\end{table}
\end{comment}

\begin{table*}[t]
\centering

\caption{\gls{DAP} results for different baselines with $c=1.5,~ p=1$.}

\label{tab:allmethods:dpog:maxcostWIthDecomposition}

\setlength{\tabcolsep}{3pt}
\begin{tabular}{l c c c c c c c }
\toprule
 & \textbf{Baseline} & Crossing & Divider & Boundary & $\downarrow$~mPLD &\meanLoc & \meanDet \\
\midrule

%\multirow{6}{*}{\rotatebox{90}{$c=1.0$}} & \multirow{3}{*}{$\Rsix$}
%& Maptrv2        & $0.959$  & $0.922$  & $0.882$  & $0.921$ & $0.182$ & $0.739$ \\
%&& StreamMapNet   & $0.929$  & $0.901$  & $0.881$  & $0.904$ & $0.242$ & $0.662$ \\
%&& MapTracker     & $\mathbf{0.880}$  & $\mathbf{0.882}$  & $\mathbf{0.843}$  & $\mathbf{0.868}$ & $0.285$ & $0.583$ \\
%\cmidrule{2-9}

%\multirow{3}{*}{$\Rhundred$}
%& Maptrv2        & $0.977$  & $0.955$  & $0.932$  & $0.955$ & $0.128$ & $0.827$ \\
%&& StreamMapNet   & $0.960$  & $0.952$  & $0.943$  & $0.952$ & $0.158$ & $0.794$ \\
%&& MapTracker     & $\mathbf{0.918}$  & $\mathbf{0.927}$  & $\mathbf{0.905}$  & $\mathbf{0.917}$ & $0.221$ & $0.696$ \\
%\bottomrule

%\multirow{6}{*}{\rotatebox{90}{$c=1.5$}}&
\multirow{3}{*}{$\Rsix$}
& Maptrv2        & $0.943$  & $0.899$  & $0.837$  & $0.893$ & $0.184$ & $0.709$ \\
& StreamMapNet   & $0.895$  & $0.870$  & $0.832$  & $0.866$ & $0.249$ & $0.617$ \\
& MapTracker     & $\mathbf{0.832}$  & $\mathbf{0.849}$  & $\mathbf{0.785}$  & $\mathbf{0.822}$ & $0.276$ & $0.546$ \\
\cmidrule{2-8}

\multirow{3}{*}{$\Rhundred$}
& Maptrv2        & $0.966$  & $0.939$  & $0.903$  & $0.936$ & $0.144$ & $0.792$ \\
& StreamMapNet   & $0.939$  & $0.932$  & $0.914$  & $0.928$ & $0.188$ & $0.740$ \\
& MapTracker     & $\mathbf{0.883}$  & $\mathbf{0.904}$  & $\mathbf{0.864}$  & $\mathbf{0.884}$ & $0.227$ & $0.657$ \\
 \bottomrule

\end{tabular}

\end{table*}

%\footnotetext{Differences in \gls{mAP} from reported values in~\cite{streammapnet,maptracker} are due to fixed-N points sampling (e.g., for $60\times30$, \streamy~and \maptracker~reports $0.339$, and $0.403$, respectively).}

\footnotetext{Differences with \cite{streammapnet,maptracker} are due to fixed‑$N$ point sampling (e.g.\ for $\Rsix$, \streamy\ and \maptracker\ report $0.341$ and $0.403$).}

%\subsection{Results from \gls{OME} methods}
\subsection{Benchmark \gls{OME} methods on nuScenes}\label{sec:exp:ome:methods}
\textbf{Setup.} We evaluate three \gls{SOTA} methods \maptrv~\cite{maptrv2}, \streamy~\cite{streammapnet}, and \maptracker~\cite{maptracker}, on the new nuScenes split of~\cite{streammapnet} that addresses geographic data leakage issue (also reported in~\cite{loc_all_eval}). We retrain \maptrv\ on this split. Predictions and ground truth are equidistantly sampled at $0.5$\,m (see discussion in Appendix~\ref{sec:sampling:effects}). 
%For fairness and compatibility with existing evaluations, we use the ground-truth annotations corresponding to each baseline (i.e., the ground truth used by each method's original evaluation protocol).
We report results at two evaluation ranges, corresponding to a near-field range $\Rsix=60\!\times\!30$\,m and a long range $\Rhundred=100\!\times\!50$\,m. 
We evaluate \gls{AP} and \gls{DAP} for pedestrian crossing, divider, and boundary. The average performance across classes is reported using \gls{mAP} and \gls{mDAP}, respectively.
The standard thresholds $\tauAP\in\{0.5,1.0,1.5\}$\,m for $\Rsix$ and $\{1.0,1.5,2.0\}$\,m for $\Rhundred$ are used for \gls{CD}-\gls{AP}. We additionally evaluate \gls{FD}-\gls{AP} using the thresholds $\tauFD\in\{1.0, 2.0, 3.0\}$\,m from OpenLaneV2~\cite{OpenLaneV22023}. Since these thresholds are already extensive, we apply them for both $\Rsix$ and $\Rhundred$. For \MultiInstMetric, since the metric already accounts for both detection and accuracy in a single score, we use a \emph{single} cutoff $c = 1.5$ for both ranges (additional values are tested in \secR~\ref{sup:sec:cutoff}). When evaluating \MultiInstMetric, we compute \gls{SOSPA} using~\eqref{eq:definition:OGOSPA} for polylines and~\eqref{eq:definition:ogospa:cyclic} for polygons (with duplicate closing points removed), with $p=1$ in both cases. To handle directionality, we evaluate \gls{SOSPA} for each instance pair under the original ordering and with one instance direction reversed, selecting the minimum value.

\textbf{Comparing methods.} \tab~\ref{tab:allmethods:ap} reports \gls{AP} and \tab~\ref{tab:allmethods:dpog:maxcostWIthDecomposition} reports \MultiInstMetric\ and its decomposition. 
Looking at \gls{FD}-\gls{AP} results, we see that despite the extensive thresholds used, the scores on $\Rhundred$ are near zero for all methods (e.g., \gls{mAP} of $0.029$, $0.031$, and $0.091$), rendering them uninformative to a practitioner seeking to differentiate method quality at longer range.
Comparing \gls{CD}-based \gls{mAP} and \gls{mDAP}, we see that the two frameworks agree on the relative ranking across both ranges. That is, \maptracker\ outperforms \streamy, which in turn outperforms \maptrv. This consistency confirms that \MultiInstMetric\ captures the essential performance differences. Beyond this agreement, the \MultiInstMetric\ decomposition reveals a structural pattern that \gls{mAP} cannot expose. That is, across all methods and both ranges, the detection error $\dNcard$ (\meanDet) dominates the localization error $\dNloc$ (\meanLoc) by a factor of two to four. Detection capability, i.e., missed instances and false alarms, is currently the primary bottleneck for \gls{OME}, not geometric precision per detected element. %For instance, at $\Rsix$ \maptracker\ achieves $\dNloc=0.276$ and $\dNcard=0.546$. 
The gap widens at the longer $\Rhundred$ range, indicating that detection coverage is harder to maintain as range increases.

%\fig~\ref{fig:qualitative} illustrates a clear example of the binary behavior of \gls{AP} across classes. For instance, \gls{AP} reports a perfect score for the boundary class despite the fact that the prediction misses a portion of the left-side road edge. Conversely, it assigns a zero score to dividers, deeming them unusable, even though predictions lie within an acceptable distance for all dividers. In contrast, \MultiInstMetric\ correctly reflects that boundary detections are not perfect and that road dividers are not entirely unusable, thereby providing a more nuanced assessment of prediction quality. The metric could be made even more accurate by allowing multiple-to-multiple assignments, in which fragmented detections or detections covering multiple ground-truth instances are permitted to match, with an appropriate penalty for instance fragmentation or merging. This limitation can be addressed in future work by adopting an approach similar to trajectory-\gls{GOSPA}~\cite{t_gospa}.

\textbf{Qualitative Comparison.} \fig~\ref{fig:qualitative} highlights the binary nature of \gls{AP} across classes. For example, \gls{AP} assigns a perfect score to the boundary class even though part of the left road edge is missed, while rating dividers as unusable despite predictions falling within an acceptable distance. In contrast, \MultiInstMetric\ more accurately reflects that boundary detections are imperfect and dividers are not entirely unusable, yielding a more meaningful assessment of prediction quality. \Gls{DAP} could be further improved by allowing multiple-to-multiple instance assignments with penalties for instance fragmentation or merging, as in trajectory-\gls{GOSPA}~\cite{t_gospa}, and is left for future work. For more qualitative examples, see Appendix~\ref{sec:additional:qualitative}.

\textbf{Runtime.} Benchmarking the evaluation time using \streamy{} on $\Rsix$, \gls{mDAP} achieves comparable runtime to \gls{CD}-based \gls{mAP}  ($\approx38$~s each with $64$~workers). This shows that \MultiInstMetric\ advantages do not come at the cost of runtime. For more details, see the analysis in Appendix~\ref{sec:sampling:effects}.

\begin{figure}
	\centering
	\includegraphics[width=0.95\textwidth]{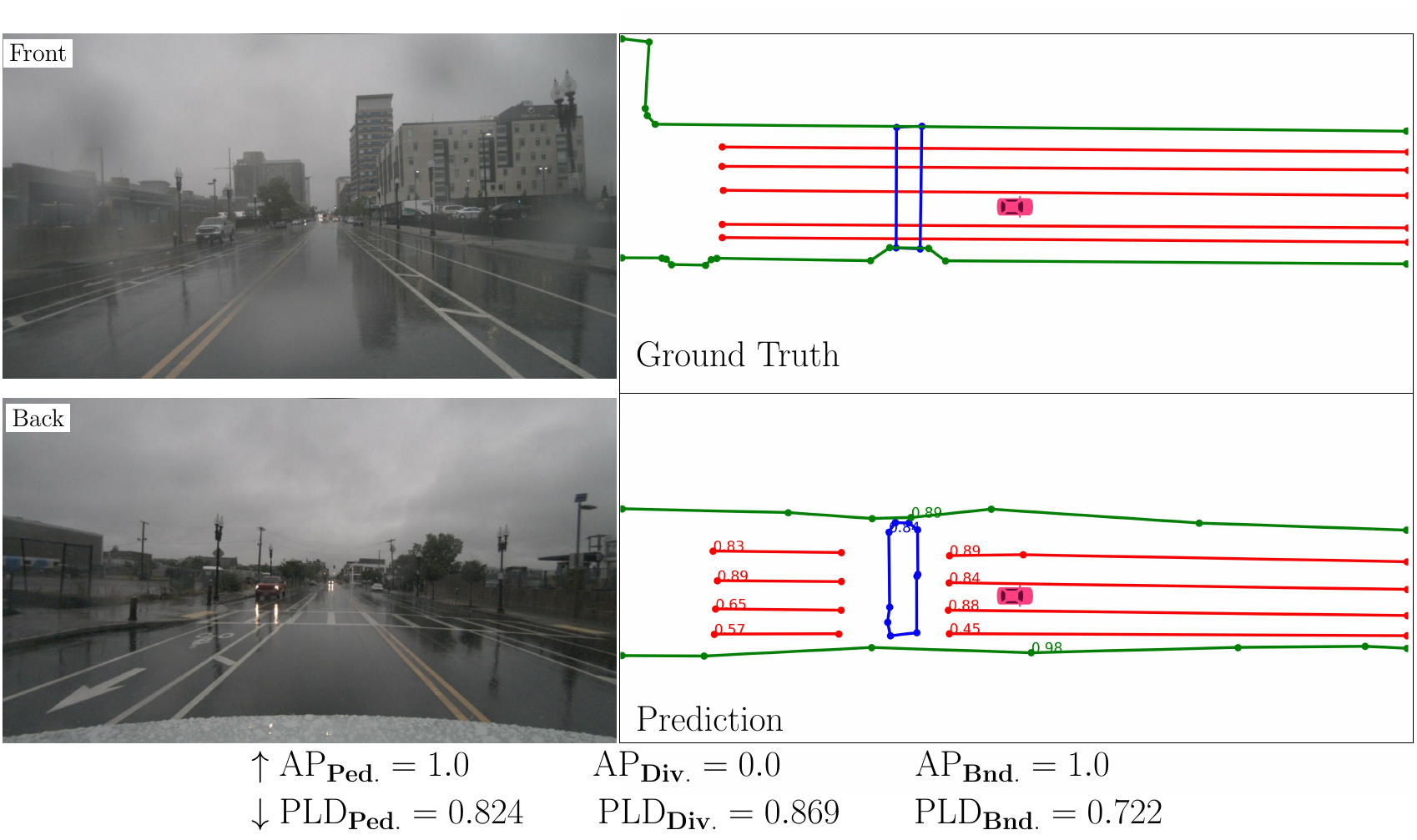}
	\vspace{-0.5em}
	\caption{Qualitative example from \maptracker{} on $\Rhundred$, comparing \gls{AP} to \gls{DAP} for pedestrian crossing (blue), divider (red), and boundary (green). ($\tauAP \in \{1.0, 1.5,2.0\}$ is used for \gls{AP}, and $c=1.5$, $p=1$ for \MultiInstMetric{}. Confidences are shown on the figure.)}
	
	\label{fig:qualitative}
    \vspace{-1.5em}
\end{figure}

%We presented \gls{SOSPA}, an order-aware polyline similarity metric, and \MultiInstMetric, a soft multi-instance evaluation metric that interprets online mapping outputs as \gls{MB} \glspl{RFS}. Both satisfy the metric axioms, are computationally tractable, and---unlike \gls{CD}--\gls{mAP}---jointly capture detection and geometric accuracy. On nuScenes, the framework preserves the relative ranking of three \gls{SOTA} methods while exposing accuracy patterns hidden by \gls{mAP}: detection capability, not geometric precision, is currently the primary limitation of \gls{OME}.

%\textbf{Limitations and future work.} Our framework focuses on geometry; instance-level topological structure (e.g.\ which lanes connect to which) is not captured. A natural extension follows graph variants such as \gls{GOSPA} for graphs~\cite{graphgospa} and lane-topology benchmarks~\cite{OpenLaneV22023}. The cutoff $c$ remains a hyperparameter: choosing it requires balancing application-driven tolerances against discriminative power across methods. Finally, while we focus on nuScenes for direct comparability with prior work, evaluating on additional datasets and on stochastic outputs would further test the framework generality.

\section{Conclusions}
To improve the understanding of \gls{OME} methods,
we introduced \gls{SOSPA}, an order-aware polyline similarity metric, and \gls{DAP}, a soft multi-instance evaluation metric, as alternatives to \gls{CD}-based \gls{mAP} for evaluating \gls{OME} methods. Both satisfy metric axioms, remain computationally tractable, and jointly capture detection and positional accuracy. 
Benchmarking three \gls{SOTA} methods on nuScenes shows that \gls{DAP} preserves relative performance rankings while its error decomposition reveals that detection error dominates across all methods, a finding that \gls{mAP} fails to capture.% by discarding positional accuracy.

\textbf{Limitations and future work.} Our framework focuses on geometry, and instance-level topological structure (e.g.\ which lanes connect to which) is not captured. One promising direction for future work is therefore to investigate topology-aware extensions of \gls{SOSPA} and \gls{DAP}~\cite{graphgospa,OpenLaneV22023}. Additionally, while we focus on nuScenes for its popularity, evaluating on additional datasets and on other methods would further test the generality of the framework.

Finally, we hope that an order-aware, decomposable, metric-based evaluation framework will support continued progress in online mapping for safer autonomous driving.

\bibliographystyle{unsrt}
\bibliography{references}
%%%%%%%%%%%%%%%%%%%%%%%%%%%%%%%%%Appendix%%%%%%%%%%%%%%%%%%%%%%%%%%%%%%%%%

\newpage
\appendix

\section{Additional experimental results}
This section presents supplementary experimental results that complement the main paper. We report the impact of polyline sampling strategies, provide a runtime comparison, analyze the sensitivity of \MultiInstMetric\ to the cutoff parameter, and include additional qualitative examples.

\subsection{Effect of sampling and runtime comparison}\label{sec:sampling:effects}
\tab~\ref{tab:sampling_effect:runtime} reports the sensitivity of \MultiInstMetric\ and its decomposition to the equidistant sampling distance, along with evaluation runtime, benchmarked using \streamy\ on $\Rsix$. \MultiInstMetric\ is stable across sampling distances, with deviations from the $0.5$~m baseline below $1\%$. The localization component $\dNloc$ shows the largest sensitivity, since denser sampling better captures geometric detail, while $\dNcard$ is virtually unaffected. For reference, the sensitivity of \MultiInstMetric\ is slightly lower than the \gls{CD}-\gls{mAP} sensitivity level. In terms of computational cost, at $0.5$~m sampling distance, \MultiInstMetric\ has comparable runtime to \gls{CD}-\gls{mAP}, with modest overhead only at finer resolutions. Hence the merits of \MultiInstMetric\ does not come at the cost of compute time. Since we argue for using \MultiInstMetric\ with only a single threshold, the current comparison against the standard three-threshold \gls{mAP} is the relevant practical scenario, where the overhead is negligible.
 % %the practical comparison is against the standard three-threshold \gls{mAP} pipeline, where the overhead is negligible.
For completeness, \tab~\ref{tab:sampling_effect:runtime} also includes results for the conventional fixed-$N$ interpolation approach ($N=200$ points), which matches the \gls{mDAP} and approximate \gls{mAP} values of the $0.5$~m equidistant baseline but incurs substantially higher runtime, with $238.2$~s for \MultiInstMetric\ and $160.1$~s for \gls{CD}-\gls{mAP}, compared to $37.8$~s and $38.0$~s, respectively at $0.5$~m equidistant sampling (a factor of approximately $4$--$6\times$ slower). This confirms that equidistant sampling is strictly preferable, as it is not only more intuitive but also yields equivalent metric values while reducing computation time by approximately one order of magnitude.

%In summary, the results in \tab~\ref{tab:sampling_effect:runtime} show that $0.5$~m is an appropriate threshold that balances accuracy with runtime. 

\subsection{Effect of \MultiInstMetric\ cutoff $c$}\label{sup:sec:cutoff}
%\textbf{Effect of cutoff $c$.} 
\tab~\ref{tab:streammapnet:threshold:effect} shows the effect of the cutoff parameter $c$ on \MultiInstMetric\ for all baselines on $\Rhundred$. Increasing $c$ from $1.0$ to $1.5$ to $2.0$ consistently reduces overall \gls{mDAP} for all methods. Looking at the decomposition of error, we see that the reduction is mostly in $\dNcard$, which is expected, since more points are assigned at this threshold. The relative ranking among baselines remains unchanged, confirming the robustness of \MultiInstMetric\ to this hyperparameter choice. 

From another perspective, at $c = 1.0$ the error is very high for all methods on $\Rhundred$, with \gls{mDAP} values exceeding $0.91$. This reflects the current state of \gls{OME} methods, whose geometric accuracy has not yet reached a level at which a $1.0$~m cutoff can meaningfully discriminate between baselines. In contrast, $c = 1.5$ provides a more balanced setting, as it is sufficiently restrictive to penalize genuine localization errors while still allowing enough assignment coverage to reveal informative differences between methods. We therefore adopt $c = 1.5$ in the experiments presented in the main text.
\color{black}
%\textbf{Category-wise analysis.}
%Pedestrian crossings consistently yield the highest \MultiInstMetric\ scores (worst performance) across all methods and ranges. This is expected given that crossings are typically polygons with complex shapes and less frequent occurrence, making them harder to detect and localize. Road boundaries exhibit the lowest scores (best performance), likely due to their simpler geometric structure and more consistent appearance.

\begin{table*}[t!]
	\centering
	\caption{Effect of sampling distance on \gls{mDAP} ($c=1.5$, $p=1$) and \gls{CD}-\gls{mAP} benchmarked using \streamy\ on $\Rsix$. The runtime is based on the average of multiple runs with $64$~workers.}
	%Percentage differences ($\Delta$) are relative to the $0.5$~m baseline.
	\label{tab:sampling_effect:runtime}
    
	\setlength{\tabcolsep}{3pt}
	\begin{tabular}{c cc cc cc cc | c c c }
		\toprule
		\multirow{2}{*}{\textbf{Samp. (m)}} &
		\multicolumn{2}{c}{$\uparrow$~\textbf{\gls{mAP}}} &
		\multicolumn{2}{c}{$\downarrow$~\textbf{\gls{mDAP}}} &
		\multicolumn{2}{c}{$\dNloc$} &
		\multicolumn{2}{c|}{$\dNcard$}  &
		\textbf{\gls{mAP}} &
		\textbf{m\MultiInstMetric} & \\
		\cmidrule(lr){2-3}\cmidrule(lr){4-5}\cmidrule(lr){6-7}\cmidrule(lr){8-9}\cmidrule(lr){10-12}
                   & val. & $\Delta$ & val. & $\Delta$ & val. & $\Delta$ & val. & $\Delta$  &  \multicolumn{2}{c}{Runtime (s)} & $\Delta$ \\
		\midrule
		$0.25$ &
		$0.345$ & ${+3.0\%}$ &
		$0.860$ & ${-0.7\%}$ &
		$0.241$ & ${-3.2\%}$ &
		$0.619$ & ${+0.3\%}$ &
		$65.9$ & $87.0$ & ${+32\%}$ \\
		$0.50$ &
		$0.335$ & --- &
		$0.866$ & --- &
		$0.249$ & --- &
		$0.617$ & --- &
		$38.0$ & $37.8$ & ${-1\%}$ \\
		$0.75$ &
		$0.323$ & ${-3.6\%}$ &
		$0.873$ & ${+0.8\%}$ &
		$0.259$ & ${+4.0\%}$ &
		$0.614$ & ${-0.5\%}$ &
		$28.5$ & $31.6$ & ${+11\%}$ \\
		\midrule
		$N\!=\!200$ &
		$0.347$ & ${+3.6\%}$ &
		$0.866$ & ${0.0\%}$ &
		$0.234$ & ${-6.0\%}$ &
		$0.632$ & ${+2.4\%}$ &
		$160.1$ & $238.2$ & ${+49\%}$ \\
		\bottomrule
	\end{tabular}
\end{table*}

\begin{comment}
    \begin{table}[!t]
\centering
\caption{Effect of cutoff parameter $c$ on \gls{mDAP} ($p=1$) for all baselines on $\Rhundred$.}
\label{tab:streammapnet:threshold:effect}

\setlength{\tabcolsep}{4pt}
\begin{tabular}{l ccc ccc}
\toprule
& \multicolumn{3}{c}{$c=1.0$} & \multicolumn{3}{c}{$c=2.0$} \\
\cmidrule(lr){2-4}\cmidrule(lr){5-7}
\textbf{Baseline} & $\downarrow$~\gls{mDAP}& \meanLoc & \meanDet & $\downarrow$~\gls{mDAP} & \meanLoc & \meanDet \\
\midrule
Maptrv2      & $0.955$ & $0.128$ & $0.827$ & $0.921$ & $0.149$ & $0.772$ \\
StreamMapNet & $0.952$ & $0.158$ & $0.794$ & $0.909$ & $0.197$ & $0.712$ \\
MapTracker   & $\mathbf{0.917}$ & $0.221$ & $0.696$ & $\mathbf{0.858}$ & $0.217$ & $0.641$ \\
\bottomrule
\end{tabular}

\end{table}
\end{comment}

\begin{table}[!t]
\centering
\caption{Effect of cutoff parameter $c$ on \gls{mDAP} ($p=1$) for all baselines on $\Rhundred$.}
\label{tab:streammapnet:threshold:effect}

\setlength{\tabcolsep}{4pt}
\begin{tabular}{l ccc ccc ccc}
\toprule
& \multicolumn{3}{c}{$c=1.0$} & \multicolumn{3}{c}{$c=1.5$} & \multicolumn{3}{c}{$c=2.0$} \\
\cmidrule(lr){2-4}\cmidrule(lr){5-7}\cmidrule(lr){8-10}
\textbf{Baseline} & $\downarrow$~\gls{mDAP}& \meanLoc & \meanDet & $\downarrow$~\gls{mDAP} & \meanLoc & \meanDet & $\downarrow$~\gls{mDAP} & \meanLoc & \meanDet \\
\midrule
Maptrv2      & $0.955$ & $0.128$ & $0.827$ & $0.936$ & $0.144$ & $0.792$ & $0.921$ & $0.149$ & $0.772$ \\
StreamMapNet & $0.952$ & $0.158$ & $0.794$ & $0.928$ & $0.188$ & $0.740$ & $0.909$ & $0.197$ & $0.712$ \\
MapTracker   & $\mathbf{0.917}$ & $0.221$ & $0.696$ & $\mathbf{0.884}$ & $0.227$ & $0.657$ & $\mathbf{0.858}$ & $0.217$ & $0.641$ \\
\bottomrule
\end{tabular}

\end{table}

\subsection{Additional qualitative examples}\label{sec:additional:qualitative}

\fig~\ref{fig:qualitative:sup:a} to \fig~\ref{fig:qualitative:sup:c} provide additional qualitative comparisons between \gls{CD}-\gls{AP} and \gls{DAP} across more scenes, complementing the example shown in the main text. The examples are generated using \streamy{} on the $\Rhundred$ evaluation range. \fig~\ref{fig:qualitative:sup:a} and \fig~\ref{fig:qualitative:sup:c} highlight cases where \gls{AP} assigns a perfect score despite noticeable geometric inaccuracies. \fig~\ref{fig:qualitative:sup:b} illustrates the failure of \gls{AP} to correctly capture partial matches, which commonly occurs for dividers. These limitations are effectively addressed by \MultiInstMetric{}, where both geometric inaccuracies and the fine-grained detection of partial matches lead to more informative scores. Hence, \gls{DAP} provides a more faithful assessment of \gls{OME} map quality.

%The following figures compare \gls{CD}-\gls{AP} and \gls{DAP} across additional scenes from \streamy{} on the $\Rhundred$ evaluation range, complementing the qualitative example in the main paper. Each scene highlights a different class-level failure mode of \gls{CD}-\gls{AP}: geometric inaccuracies and partial matches that fall within the binary matching threshold are awarded full credit, masking prediction quality that \gls{DAP} is able to distinguish. These failure cases are particularly relevant within the \gls{OME} framework, where fine-grained geometric accuracy and partial-match penalization are essential for a faithful assessment of map quality.

\begin{figure}[h]
	\centering
    \includegraphics[width=0.95\textwidth]{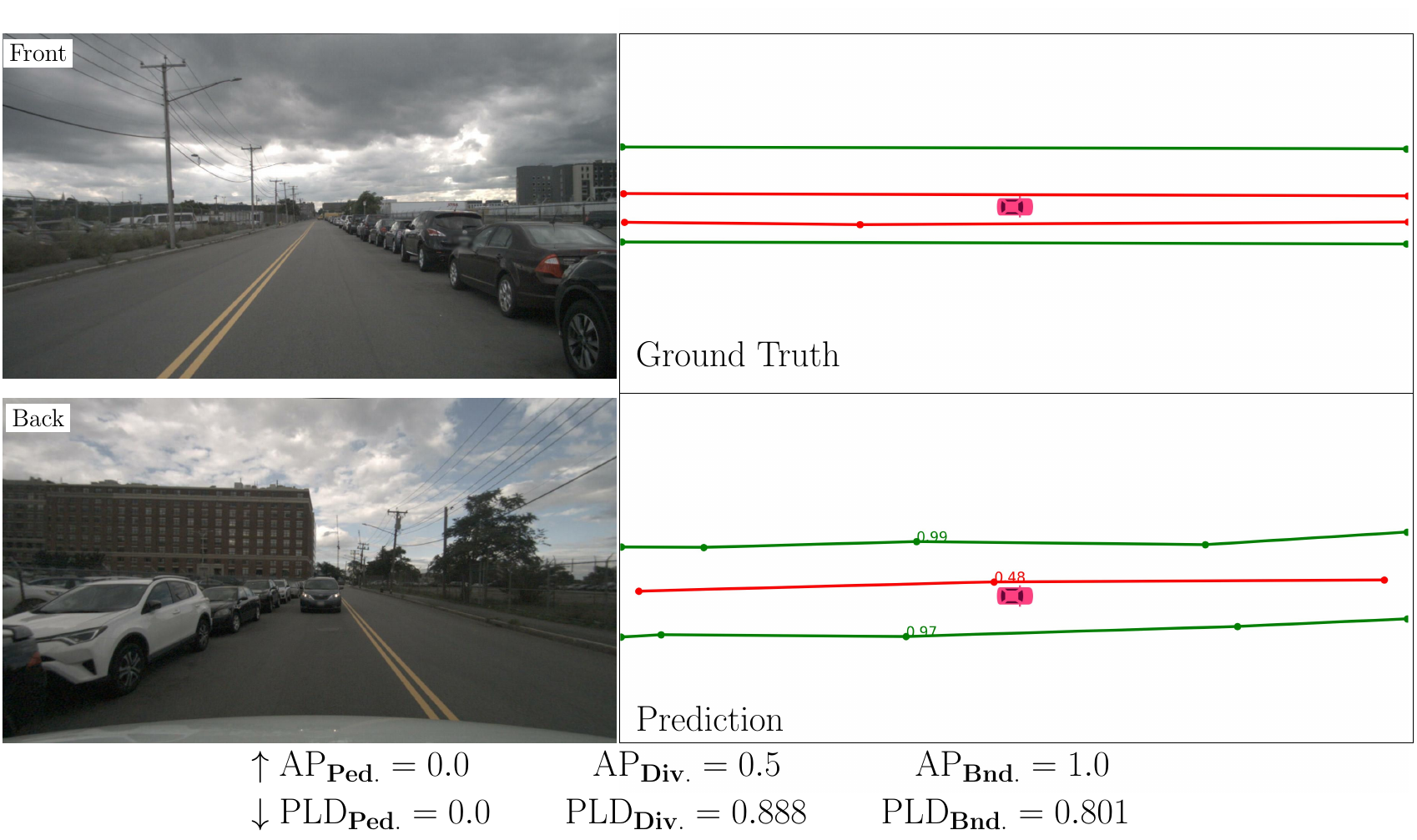}
	\caption{Additional qualitative example from \streamy{} on $\Rhundred$ ($\tauAP \in \{1.0, 1.5,2.0\}$ for \gls{AP}, and $c=1.5$, $p=1$ for \MultiInstMetric{}. Confidences shown on the figure.)}
	\label{fig:qualitative:sup:a}
\end{figure}

\begin{figure}[h]
	\centering
    \includegraphics[width=0.95\textwidth]{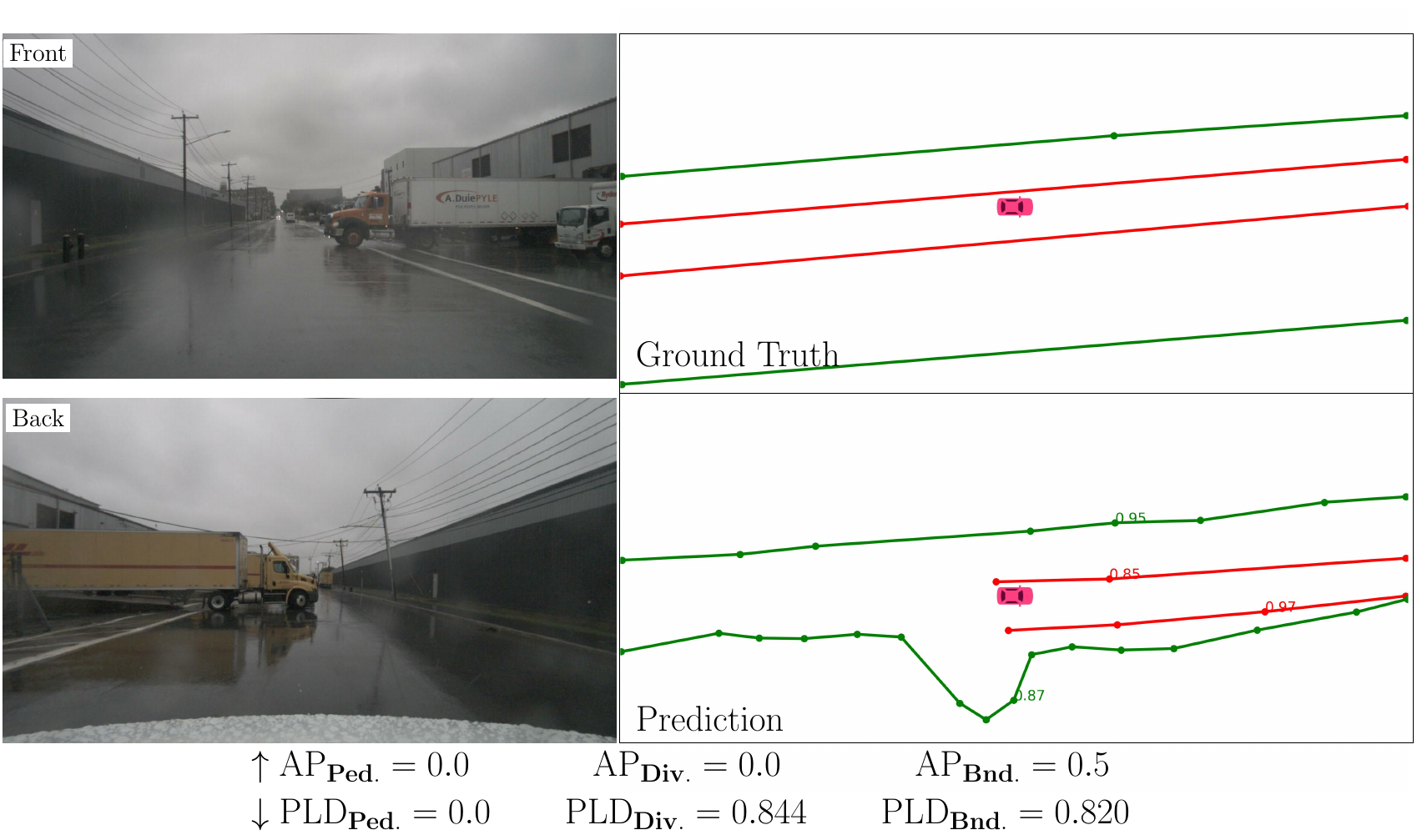}
	\caption{Additional qualitative example from \streamy{} on $\Rhundred$ ($\tauAP \in \{1.0, 1.5,2.0\}$ for \gls{AP}, and $c=1.5$, $p=1$ for \MultiInstMetric{}. Confidences shown on the figure.)}
	\label{fig:qualitative:sup:b}
\end{figure}

\begin{figure}[h]
	\centering
    \includegraphics[width=0.95\textwidth]{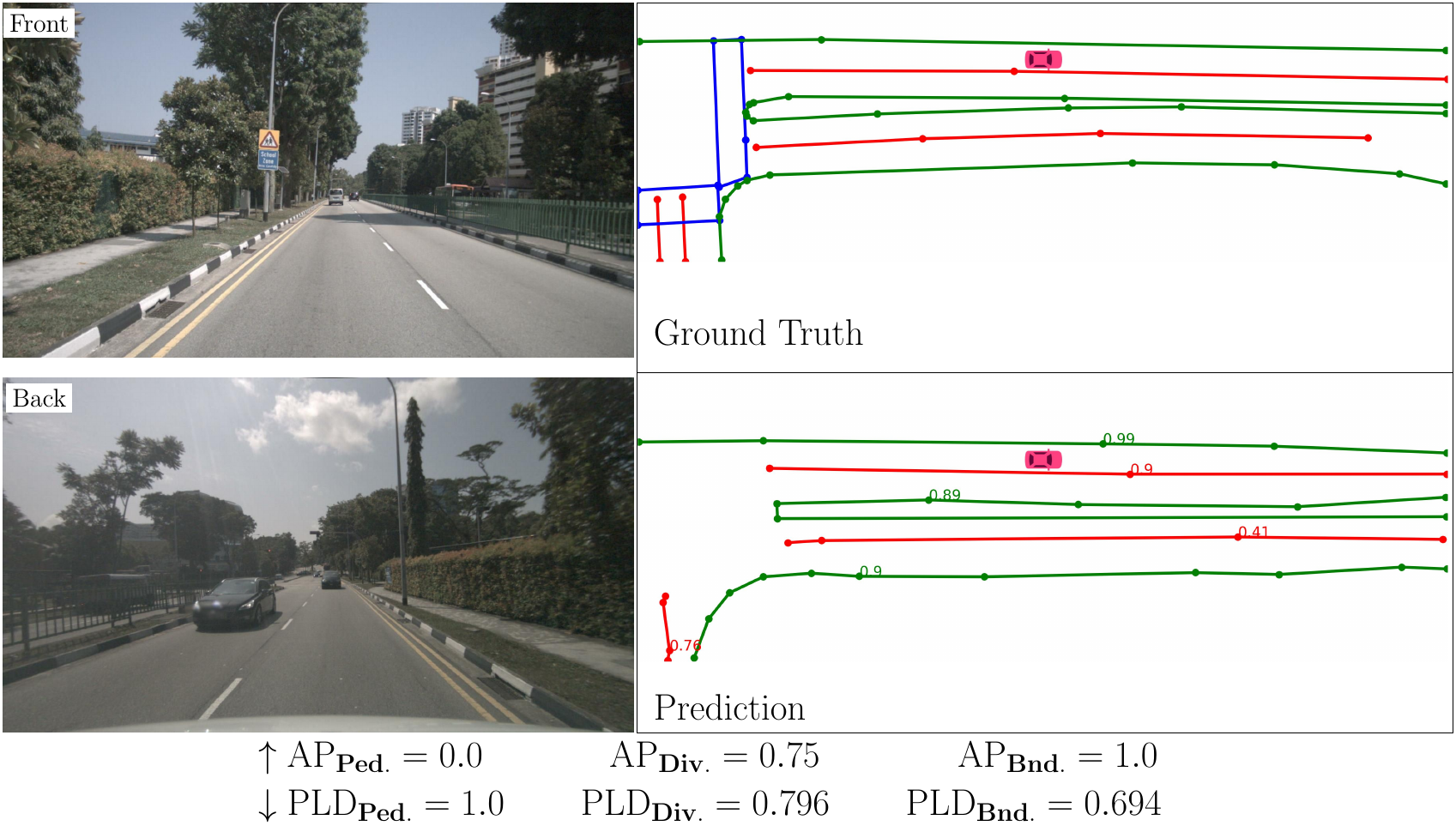}
	\caption{Additional qualitative example from \streamy{} on $\Rhundred$ ($\tauAP \in \{1.0, 1.5,2.0\}$ for \gls{AP}, and $c=1.5$, $p=1$ for \MultiInstMetric{}. Confidences shown on the figure.)}
	\label{fig:qualitative:sup:c}
\end{figure}

%===================================================Proof==========================================%
\section{Proof of the triangle inequlaity for \gls{SOSPA}}\label{appendix:sec:ogospa:traingle:proof}
For an ordered assignment set $\theta \in \AllAssignSetOrdered_{\nx,\ny}$ and finite sequneces $\x$, $\y$, define the assignment cost
		\begin{align}\label{eq:assignment:cost}
		C(\theta, \x, \y) = 
			\left(\sum_{(i,j) \in \theta} d(x_i, y_j)^p + \frac{c^p}{2}(\nx + \ny - 2|\theta|)\right)^{1/p}
		\end{align}
		so that 
        \begin{align}\label{eq:ogospa:relation:C}
            \dogospa(\x, \y) = \min_{\theta \in \AllAssignSetOrdered_{\nx,\ny}}  C(\theta, \x, \y)
        \end{align}

\subsection{Proof of triangle inequlaity for non-normalized \gls{SOSPA}}
%To show the triangle inequality for \gls{SOSPA}, we start by first showing the triangle inequality for unnormalized \gls{SOSPA}.

\begin{theorem}\label{thm:ogospa:non_normalize:triangle}
The non-normalized \gls{SOSPA} defined in~\eqref{eq:definition:OGOSPA}
satisfies the triangle inequality
\begin{align}
   \dogospa(\x, \y) \leq \dogospa(\x, \z) +\dogospa(\z, \y)
\end{align}
\end{theorem}

To show \theo~\ref{thm:ogospa:non_normalize:triangle}, we begin by stating and proving several lemmas.

Let $\theta_{\x,\z} \in \AllAssignSetOrdered_{\nx,\nz}$ and $\theta_{\z,\y} \in \AllAssignSetOrdered_{\nz,\ny}$ be ordered assignment sets. Then we define the composed set $\SetComposed$ as
\begin{align}\label{eq:composed:set:definition}
	\SetComposed \triangleq \theta_{\x,\z} \circ \theta_{\z,\y}  = \{(i,m):& (i,j) \in \theta_{\x,\z} \text{  and  }  (j,m) \in \theta_{\z,\y} \nl
    & \qquad  \text{ for some } j  \}
\end{align}
Moreover, we define
\begin{align}
	J_{\x,\z} &= \{ j : (i,j) \in \theta_{\x,\z} \text{ for some } i \}  \nl
	\Jzy &= \{ j : (j,m) \in \theta_{\z,\y} \text{ for some } m \} \nl
	\JComposed &= J_{\x,\z} \cap \Jzy  \label{eq:jxz:jzy:jc}
\end{align}
Observe that $J_{\x,\z}$ and $\Jzy$ represents indices in $\z$ matched to $\x$ and $\y$, respectively.

\begin{lemma}\label{lemma:composed:ordered}
The composed set $\theta_c$ defined in~\eqref{eq:composed:set:definition} is an ordered assignment set between $\x$ amd $\y$, i.e., $\theta_c \in \AllAssignSetOrdered_{\nx,\ny}$.
\end{lemma}

\begin{proof}
Suppose $(i_1, m_1), (i_2, m_2) \in \theta_c$ with $i_1 < i_2$. By the definition of $\theta_c$, there exist indices $j_1, j_2$ such that:
\begin{align}
    (i_1, j_1) \in \theta_{\x,\z}, \quad (j_1, m_1) \in \theta_{\z,\y}, \nl
    (i_2, j_2) \in \theta_{\x,\z}, \quad (j_2, m_2) \in \theta_{\z,\y}.
\end{align}

Since $\theta_{\x,\z} \in \AllAssignSetOrdered_{\nx,\nz}$ is an ordered set and $i_1 < i_2$, the ordering constraint implies $j_1 < j_2$.

Since $\theta_{\z,\y} \in \AllAssignSetOrdered_{\nz,\ny}$ is an ordered set and $j_1 < j_2$, the ordering constraint implies $m_1 < m_2$.

Therefore, whenever $i_1 < i_2$ for pairs in $\theta_c$, we have $m_1 < m_2$, which means $\theta_c$ satisfies the ordering constraint. Hence, $\theta_c \in \AllAssignSetOrdered_{\nx,\ny}$.
\end{proof}

\begin{lemma}\label{lemma:composed:cardinality}
Let $\SetComposed$ be as defined in~\eqref{eq:composed:set:definition} then it follows that
\begin{align}\label{eq:cardinality:inequality}
    |\theta_{\x,\z}| + |\theta_{\z,\y}| \leq \nz + |\theta_c|
\end{align}
\end{lemma}

\begin{proof}
	Given $\Jxz$, $\Jzy$, and $\JComposed$ as defined in~\eqref{eq:jxz:jzy:jc}, using the inclusion-exclusion principle we have
\begin{align}
    |\Jxz \cup \Jzy| &= |\Jxz| + |\Jzy| - |\Jxz \cap \Jzy| \nl
    &=  |\theta_{\x,\z}| + |\theta_{\z,\y}| - |\theta_c|
\end{align}
Combining this with the fact that $|\Jxz \cup \Jzy| \leq \nz$ we establish~\eqref {eq:cardinality:inequality} and the lemma follows.
\end{proof}

\begin{lemma}\label{lemma:alignment:composition}
Let $\x$, $\y$, $\z$ be finite sequences. There exists a valid ordered alignment $\theta_{\x,\y} \in \AllAssignSetOrdered_{\nx,\ny}$ such that the alignment cost~\eqref{eq:assignment:cost} satisfies
\begin{align}
    C(\theta_{\x,\y}, \x, \y) \leq \dogospa(\x, \z) + \dogospa(\z, \y) \label{eq:ogospa:triangle:anyset}
\end{align}
\end{lemma}

\begin{proof}

Given $\Jxz$, $\Jzy$, and $\JComposed$ as defined in~\eqref{eq:jxz:jzy:jc}, let $n=|\Setxz| + |\Setzy| - |\SetComposed|$, then let us define two sequnces $\mathbf{u}, \mathbf{v} \in \mathbb{R}_{\geq 0}^{n+2}$ as
\begin{align}
(u_k, v_k) =%&= \nl 
&\begin{cases}
\big(d(x_{\sigma(j)},z_j),\; d(z_j,y_{\pi(j)})\big), & k \leftrightarrow j \in J_c, \\
\big(d(x_{\sigma(j)},z_j),\; 0\big), & k \leftrightarrow j \in J_{\x,\z}\setminus J_c, \\
\big(0,\; d(z_l,y_{\pi(l)})\big), & k \leftrightarrow l \in J_{\z,\y}\setminus J_c, \\
\big(g_{\x,\z}^{1/p},\; 0\big) \text{ or } \big(0,\; g_{\z,\y}^{1/p}\big), & k = n{+}1,\, n{+}2,
\end{cases}
\end{align}
where $g_{\x,\z} = \tfrac{c^p}{2}(\nx{+}\nz{-}2|\Setxz|)$ and $g_{\z,\y} = \tfrac{c^p}{2}(\nz{+}\ny{-}2|\Setzy|)$, and $\sigma$, $\pi$ are ordered mappings on $\{1, 2,\ldots, \nx\}$, and $\{1, 2,\ldots, \ny\}$, respectively. They satisfy  
\begin{align}
    \{(\sigma(j),\pi(j)) :j \in \JComposed \} = \SetComposed 
\end{align}

By construction, and using~\eqref{eq:assignment:cost}, we have that
\begin{align}
    C(\theta_{\x,\z}, \x, \z) + C(\theta_{\z,\y}, \z, \y) = \|\mathbf{u}\|_p + \|\mathbf{v}\|_p \geq \|\mathbf{u}+\mathbf{v}\|_p. \label{eq:dxz:dzy:u+v}
\end{align}
where the inequality is obtained by applying, Minkowski's inequality ($1\leq p$).
Expanding the terms, we arrive at
\begin{align}
    \|\mathbf{u}+\mathbf{v}\|_p^p &=\sum_{j\in \JComposed}\big(d(x_{\sigma(j)},z_j) + d(z_j,y_{\pi(j)}))^p %\nl  & \qquad
    +\underbrace{A}_{\geq\, 0} + g_{\x,\z}+g_{\z,\y} \\
    &\geq \sum_{(i,j) \in \SetComposed} d(x_i,y_j)^p + \frac{c^p}{2} (\nx+\ny-2|\SetComposed|) +%\nl & \qquad 
    c^p(\nz+ |\SetComposed| - |\Setxz| - |\Setzy|) \label{eq:ogopsa:proof:basemetric:trinagle} \\
    &\geq \sum_{(i,j) \in \SetComposed} d(x_i,y_j)^p + \frac{c^p}{2} (\nx+\ny-2|\SetComposed|) \label{eq:ogogpsa:proof:u+v:inequlaity}
\end{align}
where~\eqref{eq:ogopsa:proof:basemetric:trinagle} follows by dropping $A\geq 0$ and applying the triangle inequality on the base metric $d$. The last inequality holds by applying the result of \lem~\ref{lemma:composed:cardinality}.

Taking the $p$-th root in~\eqref{eq:ogogpsa:proof:u+v:inequlaity}, and combining with~\eqref{eq:dxz:dzy:u+v}, we obtain
\begin{align}
    C(\theta_c, \x, \y) \leq C(\theta_{\x,\z}, \x, \z) + C(\theta_{\z,\y}, \z, \y)
\end{align}
for any valid ordered assignments $\theta_{\x,\z}$ and $\theta_{\z,\y}$. Choosing the optimal assignments, i.e., $C(\theta_{\x,\z}^\star, \x, \z) = \dogospa(\x,\z)$ and $C(\theta_{\z,\y}^\star, \z, \y) = \dogospa(\z,\y)$, we establish~\eqref{eq:ogospa:triangle:anyset}, since $\SetComposed \in \AllAssignSetOrdered_{\x,\y}$ is a valid ordered assignment by \lem~\ref{lemma:composed:ordered}.
\end{proof}

Now the result of~\theo~\ref{thm:ogospa:non_normalize:triangle} follows immediately by the result of \lem~\ref{lemma:composed:ordered} and~\eqref{eq:ogospa:relation:C}.

\subsection{Proof of the triangle inequality for \gls{SOSPA} normalization}\label{appendix:sec:normalization:proof}
	 The proof of triangle inequality for $\dNogospa$ is similar to that of $\dNpgospa$. We therefore demonstrate only the proof for this latter.

	 \section{Proof of the triangle inequality for  \MultiInstMetric~normalization ($p=1$)}\label{appendix:sec:normalization:proof:dap}
	We prove that $\dNpgospa$ is a metric for $p=1$. The proof follows the normalization framework of Li and Liu~\cite{Li2007NormalizedLevenshtein}. We proceed in three steps: (i) we define a \MultiInstMetric\ similarity measure analogous to the generalized Levenshtein similarity (GLS); (ii) we establish four key properties of this similarity (Theorem~\ref{thm:gls:dap:properties:new}); and (iii) we use these properties to prove that the normalized distance $\dNpgospa$ satisfies the triangle inequality and is therefore a metric (Theorem~\ref{thm:normalized:dap:metric}).
	
	%Since $\dpgospa$ is a metric for $p=1$, and the base single-instance metric $\dNogospa$ used in its definition is itself a metric (as established in the main document, Definition~\ref{definitation:normalized:ogospa}), the prerequisites for the normalization argument are satisfied.

	\begin{definition}[\MultiInstMetric\ Similarity]\label{definition:gls:dap:new}
		For $p=1$, we define \MultiInstMetric\ similarity between two sets of polylines $\X$ and $\Y$ is
		\begin{align}\label{eq:gls:dap:definition:new}
			\Spgospa(\X, \Y) = \frac{\frac{1}{2}\bigl(\EnX + \EnY\bigr) - \dpgospa(\X, \Y)}{2}.
		\end{align}
	\end{definition}

	\begin{theorem}[GLS properties for DAP]\label{thm:gls:dap:properties:new}
		Let $p=1$. For any three sets of polylines $\X$, $\Y$, $\Z$, the GLS from Definition~\ref{definition:gls:dap:new} satisfies:
		\begin{enumerate}[label=(\roman*)]
			\item $\Spgospa(\X, \X) = \dfrac{1}{2}\, \EnX$,
			\item $\Spgospa(\X, \Y) = \Spgospa(\Y, \X)$,
			\item $0 \leq \Spgospa(\X, \Y) \leq \min\bigl\{\Spgospa(\X,\X),\; \Spgospa(\Y,\Y)\bigr\} = \dfrac{1}{2}\min\{\EnX, \EnY\}$,
			\item $\Spgospa(\Y,\Y) + \Spgospa(\X,\Z) \geq \Spgospa(\X,\Y) + \Spgospa(\Y,\Z)$.
		\end{enumerate}
		These are the direct analogues of properties (1)--(4) in Theorem~1 of~\cite{Li2007NormalizedLevenshtein}, with $\lambda = 1/2$ and $|\cdot|$ replaced by $R_{(\cdot)}$.
	\end{theorem}
	
	\begin{proof}
		Before we show the properties we first establish few basic results.
			\begin{lemma}\label{lem:dap:upperbound:general:new}
			For $p \geq 1$, the \gls{DAP} metric satisfies:
			\begin{align}\label{eq:dap:upperbound:general:new}
				\dpgospa(\X, \Y) \leq \left(\frac{1}{2} (\EnX + \EnY)\right)^{1/p}
			\end{align}
		\end{lemma}
		
		\begin{proof}
			The bound is obtained by using the empty assignment $\SetUnordered = \emptyset$ in Definition~\ref{definition:pgospa}. With no matched pairs, all elements are unassigned, and the bound~\eqref{eq:dap:upperbound:general:new} follows.
		\end{proof}
		
		\begin{lemma}\label{lem:dap:lowerbound:general:new}
			For $p \geq 1$, the \gls{DAP} metric satisfies:
			\begin{align}\label{eq:dap:lowerbound:general:new}
				\dpgospa(\X, \Y) \geq \left(\frac{1}{2} \left|\EnX -\EnY \right|\right)^{1/p}
			\end{align}
		\end{lemma}
		
		\begin{proof}
			Recall Definition~\ref{definition:pgospa}. Since $\dNogospa(\x_i, \y_j) \geq 0$ and $\min(r_{\x_i}, r_{\y_j}) \geq 0$, we can lower bound by ignoring the non-negative localization terms	
			\begin{align}
				\dpgospa(\X, \Y)^p &\geq \frac{1}{2}
				\min_{\SetUnordered \in \Gamma} 
				\Bigg(
				\sum_{(i,j) \in \SetUnordered} 
				\left|r_{\x_i} - r_{\y_j} \right|
				+ 
				\sum_{\substack{i : \forall j, (i,j) \notin \SetUnordered}} r_{\x_i} %\nl& \qquad \qquad 
				+ \sum_{\substack{j : \forall i, (i,j) \notin \SetUnordered}} r_{\y_j}
				\Bigg) \label{eq:proof:general:step1:new}
			\end{align}
			
			Recalling $R_{\X} = \sum_i r_{\x_i}$ and $R_{\Y} = \sum_j r_{\y_j}$. Using the identity $|r_{\x_i} - r_{\y_j}| = r_{\x_i} + r_{\y_j} - 2\min(r_{\x_i}, r_{\y_j})$, \eqref{eq:proof:general:step1:new} simplifies to
			\begin{align}
				\dpgospa(\X, \Y)^p &\geq \frac{1}{2}
				\min_{\SetUnordered \in \Gamma} 
				\Bigg(
				R_{\X} + R_{\Y}
				- 2\sum_{(i,j) \in \SetUnordered} \min(r_{\x_i}, r_{\y_j})
				\Bigg) \label{eq:proof:general:step2:new} \nonumber
			\end{align}
			
			The term $\sum_{(i,j) \in \SetUnordered} \min(r_{\x_i}, r_{\y_j})$ is maximized when we match elements optimally, but we can only match up to $\min(R_{\X}, R_{\Y})$ total probability. Therefore,
			\begin{align}
				\dpgospa(\X, \Y)^p &\geq \frac{1}{2} \left( R_{\X} + R_{\Y} - 2\min(R_{\X}, R_{\Y}) \right) \nl
				&= \frac{1}{2} \left| R_{\X} - R_{\Y} \right|
			\end{align}
			
			Taking the $p$-th root of both sides yields~\eqref{eq:dap:lowerbound:general:new}.
		\end{proof}
		
		Now we can readily show the properties of \MultiInstMetric\ similarity.
		
		\textbf{Property (i).}
		Follows directly from %the fact that 
        $\dpgospa(\X, \X) = 0$.
		
		\textbf{Property (ii).}
		Follows from $\dpgospa(\X, \Y) = \dpgospa(\Y, \X)$.
		
		\textbf{Property (iii).}
		\emph{Lower bound.}
		From \lem~\ref{lem:dap:upperbound:general:new}, it follows that $\Spgospa(\X, \Y) \geq 0$.
		
		\emph{Upper bound.}
		By Lemma~\ref{lem:dap:lowerbound:general:new} with $p=1$, $\dpgospa(\X,\Y) \geq \tfrac{1}{2}|\EnX - \EnY|$. Substituting:
		\begin{align}
			\Spgospa(\X, \Y)
			&\leq \frac{\frac{1}{2}(\EnX + \EnY) - \frac{1}{2}|\EnX - \EnY|}{2} \nl
			%&= \frac{1}{2} \cdot \frac{\EnX + \EnY - |\EnX - \EnY|}{2} \nl
			&= \frac{1}{2}\min\{\EnX, \EnY\}.
		\end{align}
		%Since $\min\{\EnX, \EnY\} \leq \EnX$ and $\min\{\EnX, \EnY\} \leq \EnY$, we obtain $\Spgospa(\X,\Y) \leq \min\{\Spgospa(\X,\X), \Spgospa(\Y,\Y)\}$.
		By property (i), we obtain $\Spgospa(\X,\Y) \leq \min\{\Spgospa(\X,\X), \Spgospa(\Y,\Y)\}$.
		
        \textbf{Property (iv).}
		By the triangle inequality for $\dpgospa$
		\begin{align}
			\dpgospa(\X, \Y) + \dpgospa(\Y, \Z) \geq \dpgospa(\X, \Z).
		\end{align}
		Expanding using~\eqref{eq:gls:dap:definition:new} and reshuffling, we reach
		\begin{align}
			\frac{1}{2}\, \EnY + \Spgospa(\X,\Z) &\geq \Spgospa(\X,\Y) + \Spgospa(\Y,\Z)
		\end{align}
		Substituting $\Spgospa(\Y,\Y) = \tfrac{1}{2} \EnY$ (from property~(i)), yields property~(iv).
	\end{proof}
	
	\begin{theorem}[Normalized \MultiInstMetric\ is a metric]\label{thm:normalized:dap:metric}
		Let $p=1$. The normalized \MultiInstMetric
		\begin{align}
			\dNpgospa(\X, \Y) = \frac{2\, \dpgospa(\X, \Y)}{\frac{1}{2}(\EnX + \EnY) + \dpgospa(\X, \Y)}
		\end{align}
		is a metric valued in $[0,1]$.
	\end{theorem}

	\begin{proof}
		Identity and symmetry follow directly from the corresponding properties of $\dpgospa$. As for the triangle inequality, the proof follows the same argument as in~\cite[Theorem~2]{Li2007NormalizedLevenshtein}.The four properties established in Theorem~\ref{thm:gls:dap:properties:new} mirror the result in~\cite[Theorem~1]{Li2007NormalizedLevenshtein}, where for string edit distance with uniform deletion/insertion cost $\lambda$, the GLS is defined as $\mathrm{GLS}(X,Y) = \tfrac{\lambda(|X|+|Y|) - d(X,Y)}{2}$. In our setting, $\lambda = 1/2$, and $\EnX$, $\EnY$ play the role of the string lengths $|X|$ and $Y$, respectively.

		Note that $\dpgospa$ is itself a metric for $p=1$, and its base single-instance metric $\dNogospa$ (Definition~\ref{definitation:normalized:ogospa}) is a metric as established in the main document. This ensures all prerequisites are met for the argument in~\cite[Theorem~2]{Li2007NormalizedLevenshtein} to hold.
	\end{proof}

\section{Proof of \lem~\ref{lemma:maes:cyclic}}\label{appendix:sec:cyclic:proof}
To show the lemma, we first establish the following result.
		
		\begin{lemma}\label{lemma:unit:shift:transfer}
			Let $\theta \in \AllAssignSetOrdered_{\nx,\ny}$ be an ordered assignment set for $(\x, \OpShiftSet_{s_y}(\y))$ for some $s_y \in \{0, \ldots, \ny - 1\}$. Then there exist an ordered assignment set $\bar{\theta} \in \AllAssignSetOrdered_{\nx,\ny}$ and an integer $l \in \{0, \ldots, \ny - 1\}$, such that
			\begin{align}\label{eq:unit:shift:transfer}
				C\big(\theta,\, \x,\, \OpShiftSet_{s_y}(\y)\big) = C\big(\bar{\theta},\, \OpShiftSet_1(\x),\, \OpShiftSet_l(\y)\big),
			\end{align}
            where $C$ is defined in~\eqref{eq:assignment:cost}.
		\end{lemma}
		
		\begin{proof}
		Let $\y' =  \OpShiftSet_{s_y}(\y)$ for some $s_y$. 	
		Let $\pi_1(\theta) = \{i : (i,j) \in \theta\}$. To show~\eqref{eq:unit:shift:transfer}, we identify two cases. The case $1 \in \pi_1(\theta)$, and the case $1 \notin \pi_1(\theta)$
			
		\textit{Case 1:} Let $(1, j_0) \in \theta$ for some $j_0$. We can partition the assignment set following
		\begin{align}\label{eq:maes:theta:partition}
			\theta = \{(1, j_0)\} \cup \theta^\star_R, \quad \text{where } \theta^\star_R &\subseteq \{2, \ldots, \nx\} \times 
			\{j_0+1, \ldots, \ny\}.
		\end{align}
		Following this, and recalling the definition of the operator $\langle k \rangle $~\eqref{eq:cyclic:bracket:notation}, we define $\bar{\theta}$ as
		\begin{align}
			\bar{\theta} &= \big\{(i - 1,\, j - j_0) : (i, j) \in \theta^\star_R\big\} \cup \big\{(\nx,\, \ny)\big\}.\nl
			&= \big\{ \big(\langle i -1 \rangle_{\nx}, \langle j -j_0 \rangle_{\ny}\big) : (i, j) \in \theta \big\}
		\end{align}
		 Observing~\eqref{eq:maes:theta:partition} and the fact that $\big(\langle i -1 \rangle_{\nx}, \langle j -j_0 \rangle_{\ny}\big) = (\nx,\ny)$ for for the pair $(1,j_0)$, it is easy to establish that $\bar{\theta} \in \AllAssignSetOrdered_{\nx,\ny}$ with $|\bar{\theta}| = |\theta|$.
		
		  Let $l = \langle s_y + j_0 \rangle_{\ny}$ and let $\y'' = \OpShiftSet_{j_0}(\y') = \OpShiftSet_l(\y)$. We use $\bar{\theta}$ as an assignment set between $\x' =  \OpShiftSet_{1}(\x)$ and $\y''$. Then, %we have for each matched pair $(i',j') \in 	\bar{\theta}$
           \begin{align}
		  	\sum_{(i',j')\in \bar{\theta}} d(x'_{i'}, y{''}_{j'})^p &= \sum_{(i',j')\in \bar{\theta}} d ( x_{\langle i'+1\rangle_{\nx}}, y'_{\langle j'+ j_0 \rangle_{\ny}})^p \nl
            &= \sum_{(i,j) \in \theta} d(x_i, y'_j)^p
		  \end{align}  
		  %\begin{align}
		 % 	d(x'_{i'}, y''_{j'}) = d ( x_{\langle i'+1\rangle}_{\nx}, y'_{\langle j'+ j_0 \rangle}_{\ny}) = d(x_i, y'_j)
		  %\end{align}
			since 
			\begin{align}
				\langle  j'+j_0 \rangle_{\ny} &= \langle \langle j-j_0 \rangle_{\ny} +j_0\rangle_{\ny} \nl
				& = \langle (j- j_0-1) \mod \ny + 1 + j_0 \rangle_{\ny}  \nl 
				&= \big( (j- j_0-1) \mod \ny  + j_0 \big)\mod \ny +1 \nl
				&= j
			\end{align}
			and similarly, $\langle i'+1\rangle_{\nx} = \langle  \langle i -1 \rangle_{\nx}   +1\rangle_{\nx}  = i$.
			Given the result above, recalling~\eqref{eq:assignment:cost}, and using the fact that $|\bar{\theta}| = |\theta|$ we can easily arrive at 
            \begin{align}
                C(\theta, \x,  \OpShiftSet_{s_y}(\y)) = C(\bar{\theta}, \OpShiftSet_1(\x), \OpShiftSet_l(\y))
            \end{align} 
            where $l = \langle s_y + j_0 \rangle_{\ny}$ and hence~\eqref{eq:unit:shift:transfer} holds for the case $(1,j_0) \in \theta$ for some $1 \leq j_0 \leq \ny$.
			
			\textit{Case 2:} The proof for the case $1 \notin  \pi(\theta)$, can be addressed by defining a mapping $\bar{\theta} = \big\{(i-1,\, j) : (i,j) \in \theta\big\}$ and following similar steps as done earlier, to show that 
			$C(\theta, \x,  \OpShiftSet_{s_y}(\y)) = C(\bar{\theta}, \OpShiftSet_1(\x), \OpShiftSet_l(\y))$, where $l= s_y$.
			 \end{proof}
			 
			 	\begin{proof}[Proof of Lemma~\ref{lemma:maes:cyclic}]
			 	We want to show that
                \begin{align}
                    \dcyclicogospa([\x],[\y]) =\dcyclicogospa(\x,[\y])
                \end{align}
                To that end, define 
			 	\begin{align}
			 		f(s_x) \triangleq \min_{s_y \in \{0,1,\ldots, \ny-1\}} \dogospa\big(\OpShiftSet_{s_x}(\x), \OpShiftSet_{s_y}(\y)\big)
			 	\end{align}
			 	so that $\dcyclicogospa([\x],[\y]) = \min_{s_x} f(s_x)$.
			 	 To prove~\eqref{eq:maes:result}, it is enough to show that $f(s_x)$ is independent of $s_x$, i.e., $f(s_x) = f(0)$ for all $s_x \in \{0, \ldots, \nx - 1\}$. We follow a similar argument to the one outlined in \cite[\lem~3.1]{maes1990cyclic}.
                 
			 	 Observe that it suffices to show $f(1) \leq f(0)$. Since the same argument applied iteratively yields 
			 	 \begin{align}\label{eq:maes:proof:step1}
			 	 	f(0) \geq f(1) \geq \cdots \geq f(\nx) = f(0)
			 	 \end{align} 
			 	 forcing equality.
			 	 
			 	Now, Let $s_y^\star$ achieve the minimum in $f(0)$ and let $\theta^\star \in \AllAssignSetOrdered_{\nx,\ny}$ be the optimal assignment for $(\x, \OpShiftSet_{s_y^\star}(\y))$, so that $f(0) = C(\theta^\star, \x, \OpShiftSet_{s_y^\star}(\y))$.
			 	
			 	By Lemma~\ref{lemma:unit:shift:transfer}, there exist $\bar{\theta} \in \AllAssignSetOrdered_{\nx,\ny}$ and $l \in \{0, \ldots, \ny - 1\}$ such that
			 	\begin{align}
			 		C\big(\theta^\star,\, \x,\, \OpShiftSet_{s_y^\star}(\y)\big) = C\big(\bar{\theta},\, \OpShiftSet_1(\x),\, \OpShiftSet_l(\y)\big).
			 	\end{align}
			 	Therefore,
			 	\begin{align}
			 		f(1) &\leq \dogospa\big(\OpShiftSet_1(\x),\, \OpShiftSet_l(\y)\big) \nl
			 		&\leq C\big(\bar{\theta},\, \OpShiftSet_1(\x),\, \OpShiftSet_l(\y)\big) = f(0).
			 	\end{align}
			 	By~\eqref{eq:maes:proof:step1}, $f(s_x) = f(0)$ for all $s_x$, yielding~\eqref{eq:maes:result}.
			 \end{proof}

\end{document}